\documentclass[11pt,letterpaper]{article}
\usepackage{amsmath,amssymb,amsthm,leftidx}
\usepackage{graphicx,color}
\usepackage[hyphens]{url}
\usepackage{dsfont}
\usepackage{booktabs,array}
\usepackage[square,sort,numbers]{natbib}
\usepackage[sort,nocompress,noadjust]{cite}
\usepackage[normalem]{ulem}
\usepackage[boxed]{algorithm}
\usepackage{mathtools}

\usepackage[noend]{algpseudocode}
\usepackage{siunitx}
\usepackage{subcaption}
\usepackage{hhline}
\usepackage{multirow}
\usepackage{longtable}
\usepackage{etoolbox}
\usepackage{enumerate}
\usepackage{wrapfig}
\usepackage{bm}

\usepackage[nameinlink,capitalize]{cleveref}

\numberwithin{equation}{section}
\newtheorem{theorem}{Theorem}[section]
\newtheorem{cor}[theorem]{Corollary}
\newtheorem{lemma}[theorem]{Lemma}
\newtheorem{remark}[theorem]{Remark}
\newtheorem{prop}[theorem]{Proposition}

\newtheorem{example}[theorem]{Example}

\newtheorem{assumption}{Assumption}



\newcommand{\cF}{\mathcal{F}}

\newcommand{\cX}{\mathcal{X}}

\newcommand{\R}{\mathbb{R}}
\newcommand{\C}{\mathbb{C}}
\newcommand{\N}{\mathbb{N}}

\newcommand{\Prob}{\mathbb{P}}
\newcommand*{\E}{\mathbb{E}}

\providecommand{\rank}{\operatorname{rank}}

\DeclareMathOperator{\logdel}{\log\tfrac{1}{\delta}}

\newcommand*{\eps}{\varepsilon}

\renewcommand{\leq}{\leqslant}
\renewcommand{\geq}{\geqslant}

\providecommand{\abs}[1]{\left\lvert#1\right\rvert}

\newcommand{\ball}{\mathbb{B}^d}

\newcommand{\lam}{\lambda}

\newcommand{\del}{\delta}
\newcommand{\sig}{\sigma}

\newcolumntype{?}{!{\vrule width 1pt}}
\newcolumntype{C}{>{\centering\arraybackslash}m{8em}}

\newcommand{\HF}{\operatorname{HF}}

\newcommand{\HS}{\operatorname{HS}}
\newcommand{\LT}{\operatorname{LT}}
\newcommand{\sign}{\operatorname{sign}}

\newcommand{\SO}{\operatorname{SO}(3)}
\newcommand{\SOn}{\operatorname{SO}(n)}

\newcommand{\Unif}{\operatorname{Unif}}
\usepackage[margin=1in]{geometry}

\makeatletter
\def\blfootnote{\gdef\@thefnmark{}\@footnotetext}
\makeatother

\begin{document}

\title{Kernel approximation on algebraic varieties}

\author{Jason M. Altschuler \and Pablo A. Parrilo}
\date{}
\maketitle
\blfootnote{The authors are with the Laboratory for Information and Decision Systems (LIDS), Massachusetts Institute of Technology, Cambridge MA 02139. Work partially supported by NSF AF 1565235, NSF Graduate Research Fellowship 1122374, and a TwoSigma PhD Fellowship.}

\begin{abstract}
	Low-rank approximation of kernels is a fundamental mathematical problem with widespread algorithmic applications. Often the kernel is restricted to an algebraic variety, e.g., in problems involving sparse or low-rank data. We show that significantly better approximations are obtainable in this setting: the rank required to achieve a given error depends on the variety's dimension rather than the ambient dimension, which is typically much larger. This is true in both high-precision and high-dimensional regimes. Our results are presented for smooth isotropic kernels, the predominant class of kernels used in applications. Our main technical insight is to approximate smooth kernels by polynomial kernels, and leverage two key properties of polynomial kernels that hold when they are restricted to a variety. First, their ranks decrease exponentially in the variety's co-dimension. Second, their maximum values are governed by their values over a small set of points. Together, our results provide a general approach for exploiting (approximate) ``algebraic structure'' in datasets in order to efficiently solve large-scale data science problems. 
\end{abstract}

\section{Introduction}\label{sec:intro}

Given a kernel $K$, domain $\cX \subset \R^d$, and accuracy $\eps > 0$, the low-rank approximation problem
is to find a kernel $K_r$ of rank $r$, for $r$ as small as possible, satisfying
\[
	\sup_{x,y \in \cX} \abs{K(x,y) - K_r(x,y)} \leq \eps.
\]
In addition to being a fundamental mathematical problem in its own right, low-rank approximation has broad algorithmic implications in data science and applied mathematics. The connection is that low-rank kernel approximations enable rapid computation of core algorithmic tasks such as the Discrete Gauss Transform, heat equation solvers~\citep{GreStr91}, optimal transport solvers~\citep{AltBacRudWee18,SolGoePey15}, and kernel methods in machine learning~\citep{RifYeoTom03,YanDurDav05,RahRec08,WilRas06}, among many others. Designing better approximations---i.e., approximations that require smaller rank $r$ for the same accuracy $\eps$---immediately translates into faster algorithms for these myriad applications.

\par This broad applicability has led to an extensive literature on low-rank approximation of kernels. Existing approaches can be roughly partitioned into two categories depending on how the rank $r$ required to achieve $\eps$ approximation accuracy scales in the problem parameters.

\begin{itemize}
	\item \underline{High-precision approaches} scale polylogarithmically in the approximation accuracy $\eps$, but exponentially in the ambient dimension $d$. A typical rate for approximating a smooth isotropic kernel over the unit\footnote{$\cX \subset \ball$ without loss of generality because rescaling the domain $\cX$ is equivalent to rescaling the kernel function.} ball $\cX = \ball$ is
	\begin{align}
		r = O\left( \log 1/\eps \right)^d,
		\label{eq:curse-of-dim}
	\end{align}
	achieved for instance by polynomial methods~\citep{GreStr91,YanDurDav05,YanDurGum03,CotKesSre11,WanLiDar18}. Details in Proposition~\ref{prop:poly-standard}.
	\item \underline{High-dimensional approaches} scale exponentially better in $d$, but exponentially worse in $\eps$. A typical rate for approximating a positive-definite, isotropic kernel $K$ over $\cX = \ball$ is
	\begin{align}
		r = O\left( d \,\frac{\log(\sig_K / \eps)}{\eps^2}  \right),
		\label{eq:rate-high-dim}
	\end{align} 
	achieved by the Random Fourier Features method~\citep{RahRec08}.
	Above, $\sig_K^2$ is the trace of the Hessian of $k(x-y) = K(x,y)$ at $0$; this is called the ``curvature of $K$''. Details in Proposition~\ref{prop:rff-standard}.
\end{itemize}

\paragraph*{Key issue: dimension dependence.} 
Both approaches have severe limitations in practice. On one hand, high-precision approaches are limited to dimensions $d\leq5$ or $10$, say, at the most. On the other hand, high-dimensional approaches cannot approximate to accuracy $\eps$ better than a couple digits of precision with ranks $r$ of practical size (typically in the hundreds or thousands)---especially if the dimension $d$ is in the hundreds or thousands.

\paragraph*{Better approximation over structured domains?} A pervasive phenomenon throughout data science is that real-world datasets often lie on ``low-dimensional domains'' $\cX$ in a high-dimensional ambient space $\R^d$. This motivates the critical hypothesis:
\begin{align*}
	\text{\emph{If $\cX \subset \R^d$ has ``effective dimension'' $d^{\star}$, }}&\text{\emph{then the dependence on the ambient dimension $d$}}
	\\ \text{\emph{in the rates~\eqref{eq:curse-of-dim} and~\eqref{eq:rate-high-dim} can be }}&\text{\emph{improved to the analogous dependence on $d^{\star}$.}}
\end{align*}
There are different ways to formalize this notion of ``effective dimension''. Previous work has focused on exploiting local differentiable structure: consider $\cX$ to be (a bounded subset of) a low-dimensional real manifold. In contrast, this paper seeks to exploit global algebraic structure: we consider $\cX$ to be (a bounded subset of) a low-dimensional real algebraic variety.

\paragraph*{Global algebraic structure vs local differentiable structure.} These two settings of varieties and manifolds are in general incomparable.
Our investigation is motivated by the opportunity that \emph{the variety setting handles many popular applications that the manifold setting cannot.}

\par Indeed, existing bounds from the manifold literature often do not apply to the variety setting because they require the domain to satisfy smoothness or curvature bounds, and do not allow for singular points, cusps, self-intersections, etc. A quintessential example is problems involving \emph{sparse data}~\citep{CotKesSre11}, in which case $\cX$ is a low-dimensional variety: the union of low-dimensional coordinate subspaces, see \S\ref{ssec:ex:uas}. This cannot be handled by previous manifold results since $0$ is a singular point of $\cX$. Another important example is problems involving \emph{low-rank matrices}, in which case the relevant domain $\cX$ is again a low-dimensional variety which cannot be handled by previous manifold results since $\cX$ is not smooth, see \S\ref{ssec:ex:lr}. These issues can be critical in practice, not merely a theoretical technicality; see Figure~\ref{fig:nystrom-sparse}.

\begin{figure}
	\begin{center}
		\includegraphics[width=0.4\textwidth]{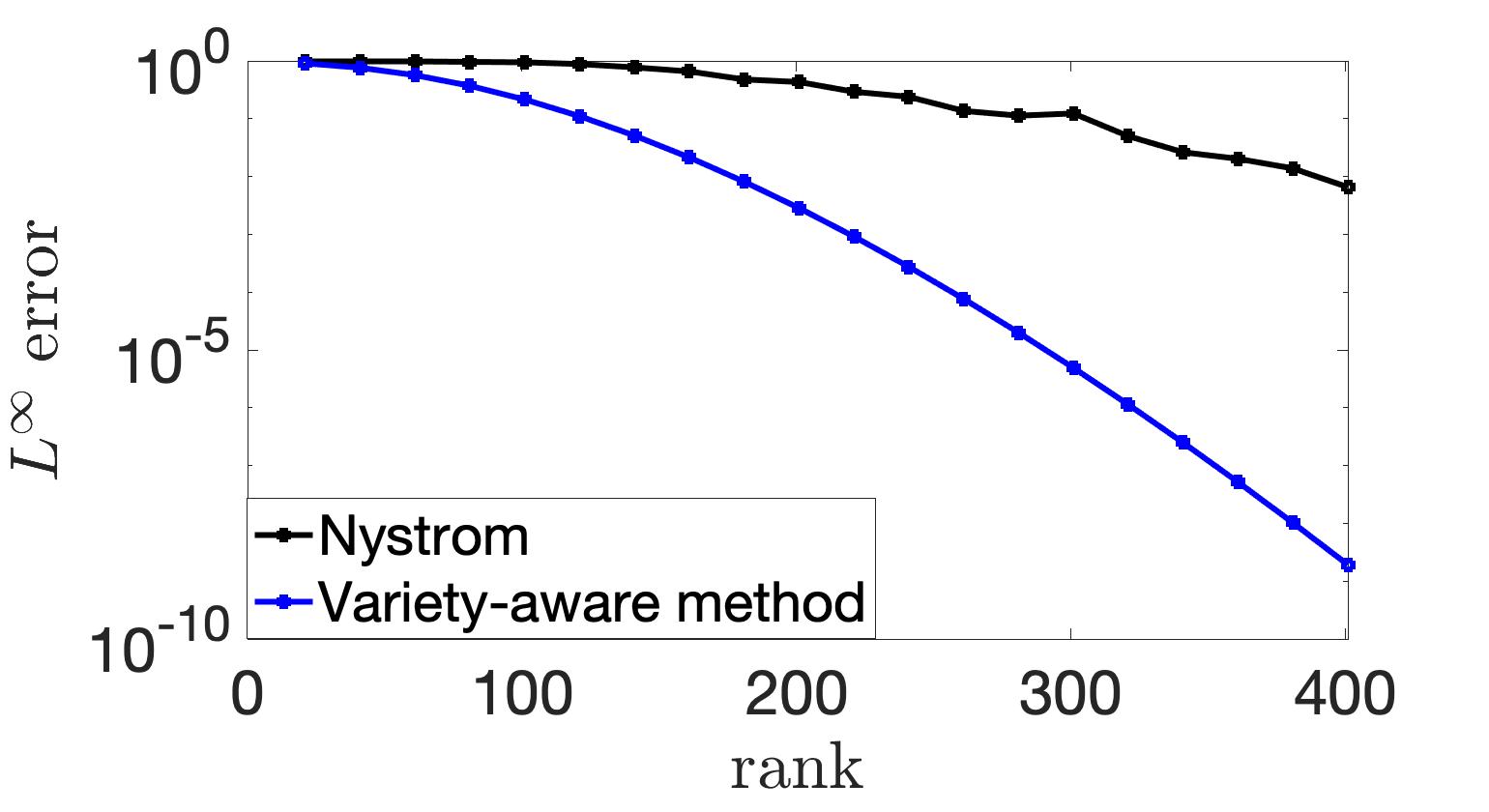}
	\end{center}
	\caption{Methods that exploit manifold structure (e.g., Nystr\"om) often \emph{cannot} exploit variety structure. This is evident even in simple settings: here we consider approximating the Gaussian kernel $e^{-\|x-y\|^2/2}$ over $1$-sparse vectors $x,y \in \R^{20}$ with norm at most $1$. Details in \S\ref{ssec:ex:uas}.}	
	\label{fig:nystrom-sparse}
\end{figure}

\subsection{Contribution: better approximation over varieties}\label{ssec:intro:cont}

This paper initiates the study of kernel approximation  in the setting that the approximation domain is a bounded subset of a real algebraic variety $V \subset \R^d$. We show that in this setting, the aforementioned hypothesis is true: for both high-precision and high-dimensional approaches, the dependence on the ambient dimension $d$ can be improved to dependence on the dimension of $V$.

\begin{theorem}[Informal version of main results]\label{thm:informal}
	Suppose $K$ is a smooth isotropic kernel with curvature $\sig_K$. Consider approximating $K$ to $L^{\infty}$ error $\eps$ over a domain $\cX = V \cap \ball$, where $V$ is a variety in $\R^d$ with dimension $d^{\star}$. 
	\begin{itemize}
		\item \underline{High-precision approach (Theorem~\ref{thm:main-var})} It suffices to have rank
		\[
			r = O\left( \log 1/\eps \right)^{d^{\star}}.
		\]
		\item \underline{High-dimensional approach (Theorem~\ref{thm:high-dim})} It suffices to have rank
		\[
		r = O\left(d^{\star} \frac{\log(\sig_K d^{\star} /\eps)}{\eps^2} \right).
		\]
	\end{itemize}
\end{theorem}

\paragraph*{Remarks.}

\begin{description}
	\item \emph{Our result gives tighter bounds for the many domains that are both manifolds and varieties.} For instance, while~\citep[Corollary 4]{AltBacRudWee18} shows an analog to our high-precision result in the case that $\cX$ is a bounded manifold with dimension $d^{\star}$, the exponent is $2.5 d^{\star}$ rather than $d^{\star}$. Because manifold dimension and variety dimension coincide when $\cX$ is both a real manifold and a real algebraic variety, Theorem~\ref{thm:main-var} provides precisely the high-precision result sought in the manifold literature: exponential dependence in the effective dimension, with no fudge factors. 
	
	\item \emph{Our techniques extend to more general kernels.} For simplicity, we state our results for smooth isotropic kernels since on one hand this captures most popular kernels used in practice, and on the other hand this level of generality yields simple proofs. Neither result requires isotropy (recall this means rotation-invariance and translation-invariance). Indeed, our high-dimensional result requires only translation-invariance, and our high-precision result assumes isotropy solely so that the ``smoothness'' assumption on the multivariate kernel $K$ can be simply stated in terms of the univariate function $f$ satisfying $K(x,y) = f(\|x-y\|^2)$. There are many ways to quantify smoothness. Our high-precision approach works whenever $K$ is well-approximated by a low-degree polynomial; this occurs for instance if $f$ is analytic in a neighborhood around the relevant domain and is satisfied for many kernels used in practice, for example the Gaussian and Cauchy kernels (see \S\ref{ssec:prelim:standard}).
	
	\item \emph{Approximation rates of course cannot depend on the domain solely through its effective dimension.} Indeed, approximation over a space-filling curve or over a union of hyperplanes in a tightly gridded formation, is effectively as difficult as approximation over a full-dimensional domain. A similar concern also applies if the domain contains arbitrary lower-dimensional components (e.g., a cloud of points), see the discussion following Lemma~\ref{lem:hf-equidim}. In the manifold setting, degenerate domains are excluded by assuming bounds on their reach or on high-order derivatives of an atlas, see e.g.,~\citep{AltBacRudWee18,baraniuk2009random}. In our algebraic setting, we make the natural assumption that the variety's \emph{degree} is not super-exponentially large: $\log \deg(V) = O(\dim V)$. This assumption is satisfied in common situations (see \S\ref{sec:ex}), and can be checked either using standard techniques if the variety is known in advance (see \S\ref{ssec:prelim:ag}) or otherwise via estimation from samples~\citep{breiding2018learning}.
	
	\item \emph{Adaptivity.} Both our results are presented from an existential point of view. An interesting algorithmic feature of the high-dimensional approach is that it automatically adapts to the variety. In contrast, the high-precision approach depends on the description of the variety.
	
	\item \emph{Kernels restricted to a variety vs kernels on a variety.} 
	Approximating an isotropic kernel $K(x,y) = f(\|x-y\|^2)$ over a domain $\cX$ can refer either to restricting $K$ to $x,y \in \cX$, or replacing $\|x-y\|$ by an intrinsic metric over $\cX$~\citep{feragen2016open}. This paper considers the former.

\end{description}

\subsection{Techniques}\label{ssec:intro:tech}

In order to prove our results, we synthesize techniques from the traditionally disparate fields of approximation theory and algebraic geometry. Key to our proofs are certain structural properties of polynomial kernels that hold when they are restricted to algebraic varieties---this may be of independent interest. We detail these properties below and how they enable us to exploit the algebraic structure of the approximation domain in the context of low-rank kernel approximation.

\paragraph*{High-precision approach.} The standard such approach exploits smoothness in order to approximate a kernel over a compact domain by a low-degree polynomial kernel. Details in \S\ref{ssec:prelim:standard}.
However, the fundamental obstacle is that the rank of a polynomial kernel grows rapidly in its degree $n$, namely as 
$\binom{n+d}{d} \approx n^d$.
This is because expressing a polynomial in $d$ variables of degree $n$ as the weighted sum of monomials (or any other basis) potentially requires all monomials in $d$ variables of degree at most $n$---of which there are precisely this many. This exponential growth in $d$ is called the ``curse of dimensionality'' for high-precision approaches.

\par The primary insight behind our high-precision approach is that if a polynomial kernel is restricted to a variety $V \subset \R^d$, then its rank drops. In fact, significantly so if $\dim(V) \ll d$. This is perhaps most easily seen in the case of sparse data~\citep{CotKesSre11}, i.e., in the case that the variety $V$ consists of all $k$-sparse points on $\R^d$. Observe that in this case, any monomial that depends on more than $k$ variables vanishes on $V$. Thus a polynomial of degree $n$ over this variety $V$ can be expressed as a weighted sum of such monomials---of which there are exactly $\sum_{i=0}^k \binom{d}{i} \binom{n}{i}$ by Proposition~\ref{prop:sparse:ip}. Critically, this grows at a much slower rate of roughly $n^k$ rather than $n^d$. Such an improvement is crucial even in small-scale settings.

\par For other varieties $V$, it is not true that monomials vanish over $V$. However, the monomials become linearly dependent. For example, if $V = \{(x_1,x_2) : x_1^2 + x_2^2 = 1\}$ is the unit circle, then although no monomials vanish on $V$, the monomials $1$, $x_1^2$, and $x_2^2$ are linearly independent when restricted to $V$. This linear dependence is the generic source of rank reduction since it enables us to use a refined description of the feature space beyond the standard one of all low-degree monomials. 
The refined description essentially\footnote{Factorizing over the cooordinate ring fully exploits the domain's algebraic structure. In \S\ref{ssec:hp-rank}, we obtain even better rank bounds by further exploiting a second source of algebraic structure: the polynomial function itself.}
 uses bounded-degree monomials in the \emph{coordinate ring} corresponding to $V$. Intuitively, this removes all ``redundancies'' in the standard description in order to produce factorizations with lower rank.
Quantitatively, we provide an exact formula in terms of the rank of a finite-dimensional matrix (Proposition~\ref{prop:rank-exact}), and show how to compute tight asymptotic rank bounds in terms of the associated Hilbert function (Proposition~\ref{prop:rank-asymp}).

\paragraph*{High-dimensional approach.} The standard such approach uses Bochner's Theorem to express a positive-definite, translation-invariant kernel as the expectation $K(x,y) = \E_{\omega} [ f_{\omega}(x) f_{\omega}(y)]$ where $f_{\omega}$ is a bounded function for any realization of the random variable $\omega$, and then approximates this by taking the empirical average of $r$ samples~\citep{RahRec08}. Details in \S\ref{ssec:prelim:standard}. Both the standard analysis and our analysis proceed in two steps:
\begin{enumerate}
	\item[(i)] Bound the approximation error on a finite subset $S$ of the compact domain $\cX$.
	\item[(ii)] Extend this approximation bound over $S$ to all of $\cX$.
\end{enumerate}
Step (i) is straightforward: a standard Chernoff bound ensures small error over any finite set $S$ with high probability if $r$ grows logarithmically in $|S|$. Step (ii) is the fundamental obstacle: in order to extend an approximation bound over $S$ to $\cX$, the set $S$ must be large. Existing analyses take $S$ to be an $\eps$-net of $\cX$ and argue step (ii) via Lipschitz smoothness of the approximation error; however the straightforward analysis requires $|S| \approx 1/\eps^d$, whereby the rank $r$ scales at least linearly in $d$.

\par We approach step (ii) from an algebraic perspective rather than an analytic one. Briefly, our primary insight is that if the approximation domain $\cX$ is a subset of a variety $V$, then the rigidity of polynomials enables us to prove step (ii) for a certain subset $S \subset \cX$ that is of size exponential in $\dim(V)$ rather than in $d$.

\par To describe our approach in more detail, it is insightful to rephrase step (ii) in the language of approximation theory. Note that although smooth isotropic kernels and the aforementioned approximate kernel are not polynomials, they are well-approximated by low-degree polynomials, thus the approximation error (their difference) is essentially a low-degree polynomial. The upshot is that then step (ii) amounts to the following central problem in approximation theory\footnote{In the jargon of approximation theory, this problem can be equivalently stated as: Find the smallest set of interpolation nodes for which the corresponding Lebesgue constant is at most $\lambda$. It can also be stated in terms of bounding the norm of an associated interpolation operator, or finding small norming sets. See e.g.,~\citep{bos2018fekete,Riv81}.}: given a compact domain $\cX$, degree $n$, and slack $\lam > 1$, find the smallest subset $S \subset \cX$ satisfying
\begin{align}
	\max_{x \in \cX} \abs{p_n(x)} \leq \lam \cdot \max_{x \in S} \abs{p_n(x)},
	\qquad \text{ for all polynomials $p_n$ of degree $n$}.
	\label{eq:intro-norming}
\end{align}
A classical result in interpolation theory is that there exists a set $S$ of size $|S| \approx n^d$ satisfying~\eqref{eq:intro-norming} for $\lambda \approx n^d$. However, this is insufficient for our purposes for two reasons: we need $|S|$ to be exponential in $\dim(V)$ rather than $d$, and $\lambda$ to be $O(1)$. As shown in Proposition~\ref{prop:norming}, the former issue is fixable by working over polynomials in the corresponding coordinate ring, and the latter issue is fixable by slightly increasing $|S|$ and appealing to a tensorization argument inspired by~\citep{bloom2012polynomial}.

\subsection{Related work}\label{ssec:intro:prev}

There is an extensive literature on low-rank approximation of kernels. This is in large part because different methods are better suited to different parameter regimes---depending on tradeoffs between the dimension, accuracy, and kernel bandwidth, as well as the structure of the data's domain (or even its distribution). We briefly overview existing approaches and their tradeoffs.

\paragraph*{High-dimensional vs high-precision approaches.} Methods that do not exploit the structure of the approximation domain can be grouped into the two categories described above: high-dimensional and high-precision approaches. 
\par \underline{High-dimensional approaches} scale to high ambient dimension and narrow kernels, but cannot provide approximations past a few digits of accuracy with ranks of practical size. This is because the ranks of these approaches scales polynomially in $1/\eps$. This is the case for Random Fourier Features~\citep{RahRec08} and its variants~\citep{liu2020random}, as well as sketching-based aproaches~\citep{ahle2020oblivious,woodruff2020near,woodruff2014sketching}. It is also the case for related approaches such as hashing-based approaches~\citep{ChaSim17,SprShr17} which target similar downstream applications without actually performing kernel approximation.
\par \underline{High-precision approaches} scale exponentially better in the accuracy $\eps$, but suffer from the curse of dimensionality. For instance, this is the case for polynomial approaches~\citep{YanDurDav05,CotKesSre11,WanLiDar18,zwicknagl2009power, schaback2008limit, wathen2015spectral,WanLiDar18,TauWec09} including Fast Multipole-esque approaches~\citep{GreStr91,greengard1987fast}, as well as the Nystr\"om method~\citep{AltBacRudWee18,WilSee01}. 
It is also worth mentioning that high-precision approaches tend to have worse dependence on the kernels' bandwidth than high-dimensional approaches. In the absence of structure, this prohibits high-precision approaches beyond low dimensions and moderately narrow kernels.

\paragraph*{Exploiting structure.} \underline{Local differentiable structure vs global algebraic structure.} As described above, better approximations are sometimes obtainable if the data lie on a ``structured'' domain. Previous work has focused on local differentiable structure, requiring the domain to be a low-dimensional smooth manifold (plus various technical assumptions for analysis purposes), see e.g.,~\citep{AltBacRudWee18,fuselier2012scattered,hangelbroek2011kernel,hangelbroek2010kernel,baraniuk2009random}. In contrast, this paper investigates domains with global algebraic structure, namely low-dimensional varieties. 
These two sources of structure sometimes coincide (e.g., $\SOn$), but sometimes only one is present (e.g., only variety structure is available for sparse or low-rank data).

\par \underline{Data adaptivity.} Different algorithms adapt to the data distribution to varying degrees. On one end of the spectrum is the Nystr\"om method, which is fully adaptive in the sense that its uses data samples to form its approximations, see e.g.,~\citep{WilSee01,GitMah16}. Other methods adapt partially, for instance Random Fourier Features can be shown to automatically adapt to the affine span of the distribution's support. Our approaches in \S\ref{sec:high-prec} and \S\ref{sec:hd} adapt to the support beyond just this affine span: indeed, in some sense, they are fully adaptive to the support when it lies on a variety.

\par \underline{Local vs global low-rank approaches.} 
An orthogonal---in fact, complementary---axis of designing kernel approximations is to exploit the spread of the data distribution in addition to the differentiable/algebraic properties of the distribution's support.
This has no effect on $L^{\infty}$ norm approximation, but can be helpful for approximation in average-case norms.
The standard approach for exploiting the data's spread is to subdivide the space (in either an a priori manner such as gridding or a data-adaptive manner such as clustering),
and then perform any of the aforementioned approximation methods locally. The basic idea is that this leverages the rapid spatial decay of kernels such as the Gaussian kernel: by zooming in, the bandwidth is effectively larger, enabling low-rank approximations to perform better locally. The result is essentially the sum of local low-rank approximations. A prominent example is the famous Fast Multipole-esque Method~\citep{GreStr91,greengard1987fast}, see also e.g.,~\citep{WanLiMah19,LeeGraMoo06,MorSriRay09,BeaGre97,SiHsiDhi17}.

\par \underline{Sparse data.} The paper~\citep{CotKesSre11} was the first to make the point that Taylor Features performs better on sparse data because some monomials vanish. However, their approach is based on the vanishing of monomials and does not generalize to other varieties. Developing high-precision approaches over other varieties requires understanding how the algebraic structure of the domain improves the rank of polynomial kernels, as shown in \S\ref{ssec:hp-rank}. Developing high-dimensional approaches over other varieties requires completely different techniques---namely norming sets for varieties, see \S\ref{sec:hd}.

\subsection{Outline}\label{ssec:intro:outline}

\S\ref{sec:prelim} introduces relevant preliminaries.
Our results on kernel approximation over varieties are presented in \S\ref{sec:high-prec} (for high-precision regimes) and in \S\ref{sec:hd} (for high-dimensional regimes). 
We illustrate our results on a number of example varieties in \S\ref{sec:ex}. We conclude in \S\ref{sec:discussion} with potential future directions.

\section{Preliminaries}\label{sec:prelim}

In this section, we briefly recall relevant background about kernels in \S\ref{ssec:prelim:kernels} and algebraic geometry in \S\ref{ssec:prelim:ag}. Further details can be found in, e.g., the standard texts~\citep{WilRas06,SchSmo02,ShaCri04} for kernels and~\citep{Cox06,Cox13,Sha94} for algebraic geometry. In \S\ref{ssec:prelim:standard}, we describe two popular low-rank approximation approaches since we build upon them in the sequel. Readers familiar with any of these three topics should feel free to skip the corresponding sections. However, we note that the way we introduce previous approaches in \S\ref{ssec:prelim:standard} is non-standard: while the literature casts high-precision and high-dimensional approaches as fundamentally different, we attempt to introduce them in a somewhat unified manner.

\par \underline{Notation.} We denote the Euclidean norm by $\|\cdot\|$, 
the $L^{\infty}$ norm over a domain $S$ by $\|\cdot\|_{S}$, the set $\{1, \dots, n\}$ by $[n]$, 
and the set of positive (resp., non-negative) integers by $\N$ (resp., $\N_0)$.
We emphasize one non-standard notation: throughout a kernel is \emph{not} necessarily PSD unless explicitly specified. 
We do this because our results in \S\ref{sec:high-prec} apply to polynomial kernels regardless of whether they are PSD, and this extension to indefinite kernels may be of interest since they are used in applications (e.g., the Fast Gauss Transform and Optimal Transport).

\subsection{Kernels}\label{ssec:prelim:kernels}
\paragraph*{Kernels and kernel matrices.} A \emph{kernel} $K : \cX \times \cX \to \R$ is a symmetric function of two arguments (not necessarily PSD, see above). 
The \textit{kernel matrix} associated to a kernel $K$ and set of points $\{x_1, \dots, x_N\} \subset \cX$ is the $N \times N$ symmetric matrix with $ij$-th entry $K(x_i,x_j)$.

\paragraph*{Types of kernels.}  A kernel is \textit{positive semidefinite} (PSD) if for any $N \in \N$ and any set $\{x_1, \dots, x_N\}$, the corresponding $N \times N$ kernel matrix is a PSD matrix. Similarly, a kernel is \emph{positive definite} (PD) if all corresponding kernel matrices are PD matrices.
A \textit{polynomial kernel} is a kernel that is also a polynomial. The \emph{degree} of a polynomial kernel is the total degree in either of its arguments $x$ or $y$ (these two numbers are the same by symmetry). 
A kernel is \textit{isotropic} if it is translation-invariant and rotation-invariant, i.e., $K(x,y)$ is a function only of $\|x-y\|$.  An \textit{isotropic polynomial kernel} is a kernel of the form $p(\|x-y\|^2)$ where $p$ is a polynomial. 

\paragraph*{Rank of kernels.} A \emph{rank-$r$ factorization} of a kernel $K : \cX \times \cX \to \R$ is a pair of ``feature maps'' $\phi, \psi : \cX \to \R^r$ such that $K(x,y) = \langle \phi(x), \psi(y) \rangle$ for all $x,y$. The \textit{rank} of a kernel is the minimum $r$ for which there exists a rank-$r$ factorization. A PSD kernel admits a minimal-rank factorization with $\phi = \psi$. Observe that the rank of a kernel cannot increase upon restriction of the kernel's domain. 
If a kernel $K$ admits a rank $r$ factorization $\langle \phi(x), \psi(y)\rangle$, then the $N \times N$ kernel matrix corresponding to any set of points $\{x_1, \dots, x_N\} \subset \cX$ admits the factorization $A^TB$ where $A,B \in \R^{r \times N}$, and the $i$-th columns of $A$ and $B$ are respectively $\phi(x_i)$ and $\psi(x_j)$.

\subsection{Algebraic geometry}\label{ssec:prelim:ag}

\paragraph*{Polynomial spaces.} 
We write $\R[x]$ to denote the polynomial ring of $d$-variate polynomials with real coefficients in the variable $x\in \R^d$. We write $\R_{n}[x]$ (resp., $\R_{\leq n}[x]$) to denote the linear space of polynomials in $\R[x]$ of degree $n$ (resp., at most $n$).

\begin{lemma}[Dimension of polynomial spaces]\label{lem:poly-dim}
	Let $x \in \R^d$.
	\begin{itemize}
		\item \underline{Bounded-degree polynomials:} $\dim(\R_{\leq n}[x]) = \binom{n+d}{n}$.
		\item \underline{Homogeneous polynomials:} $\dim(\R_n[x]) = \binom{n+d-1}{n}$.
	\end{itemize}
\end{lemma}

\paragraph*{Varieties and ideals.} Throughout, $V$ is an (affine) \emph{real algebraic variety} in $\R^d$---or, variety for short---meaning that it is the set of points in $\R^d$ on which a set of polynomials in $\R[x]$ vanishes. The associated ideal of all vanishing polynomials is $I(V) = \{p \in \R[x] : p(x) = 0, \forall x \in V \}$. This is always a real radical ideal; in what follows, every ideal $I$ considered is real radical. The vanishing set for an ideal $I \subset \R[x]$ is the variety $V(I) = \{x \in \R^d : p(x) = 0, \forall p \in I\}$. Let $\R[V]$ denote the set of equivalence classes of polynomials in $\R[x]$, where two polynomials are identified if their restrictions to $V$ are identical. By the Real Nullstellensatz Theorem, $\R[V]$ is isomorphic to the coordinate ring $\R[x]/I(V)$. The spaces $\R[x]$ and $\R[V]$ form real vector spaces in the natural way.

\paragraph*{Dimension.} There are several equivalent definitions of the \emph{dimension} $\dim V$ of a variety $V$. An intuitive geometric definition is the maximum $d^{\star}$ for which there exists a sequence $V_0 \subsetneq V_1 \subsetneq \dots \subsetneq V_{d^{\star}}$ of irreducible subvarieties of $V$. Note that unlike manifolds, the dimension of a variety $V$ is not simply the dimension of the tangent space at any point $x \in V$---in fact, these tangent spaces might have dimension different from $\dim V$, or even be undefined altogether if the variety is not smooth at that point. For instance, the union of a disjoint plane and line is a $2$-dimensional variety. A variety is \emph{equidimensional} if each irreducible component has the same dimension.

\paragraph*{Degree.} The \emph{degree} $\deg V$ is the number of intersections over $\C^d$ (counted with intersection multiplicity) of $V$ with a subspace of co-dimension $\dim(V)$ in general position.

\paragraph*{Hilbertian quantities.} Let $I_{\leq n}(V)$ denote the set of polynomials of degree at most $n$ in the vanishing ideal $I(V)$. This is a vector subspace of $\R_{\leq n}[x]$. We denote the quotient space $\R_{\leq n}[x] / I_{\leq n}(V)$ by $\R_{\leq n}[V]$.
The \emph{Hilbert function}\footnote{This is sometimes called the affine Hilbert function to distinguish it from the projective Hilbert function. Similarly for the Hilbert series. We drop the word ``affine'' throughout since there is no confusion.} of $V$ is the function
$\HF_{V}  : \N_0 \to \N_0$ defined by 
\begin{align}
	\HF_V(n)
	= \dim\left( \R_{\leq n}[V] \right),
	\label{eq:HF}
\end{align}
where the notion of dimension here is the one for vector spaces.
The \emph{Hilbert series} (a.k.a., Hilbert-Poincar\'e series) of $V$ is the generating function
\begin{align}
	\HS_{V}(t) = \sum_{n=0}^{\infty} \HF_{V}(n) t^n,
	\label{eq:HS}
\end{align}
viewed as a formal power series. The Hilbert function and series of a variety $V$ are identical to the corresponding Hilbert function and series for its vanishing ideal $I(V)$. 

\begin{lemma}[Dimension and degree in Hilbertian quantities]\label{lem:hilbert}
	Let $V$ be a variety.
	\begin{itemize}
		\item The Hilbert function $\HF_V(n)$ is a polynomial in $n$ for all sufficiently large $n$. This polynomial has degree $\dim V$ and leading coefficient $\deg V / \dim V!$. 
		\item The Hilbert series $\HS_V(t)$ can be expressed as a rational function in $t$ of the form $p_V(t)/(1-t)^{\dim V + 1}$, where $p_V$ is a polynomial with integer coefficients satisfying $p_V(1) = \deg V$.
	\end{itemize}
\end{lemma}

\begin{lemma}~\citep{chardin1989majoration}\label{lem:hf-equidim}
	If $V$ is an equidimensional variety, then
	$\HF_V(n) \leq \deg V \binom{n + \dim V}{\dim V}$.
\end{lemma}

Throughout we present our results for equidimensional varieties; this assumption holds in all example varieties in this paper, and can be removed without changing the asymptotics in our results. Equidimensionality let us non-asymptotically bound the Hilbert Function via Lemma~\ref{lem:hf-equidim}; nevertheless, the same asymptotics $\HF_V(n) = \deg(V) n^{\dim V} + O(n^{\dim V-1})$ hold in general by Lemma~\ref{lem:hilbert}. The difference is that without equidimensionality, the non-asymptotic (a.k.a. transient) behavior of $\HF_V(n)$ can change since the lower-order terms 
can depend on low-dimensional components of $V$. For instance, if $V$ is a line unioned with many points, then the dimension and degree of $V$ are dictated by the line (i.e., $\dim V = 1$ and $\deg V = 1$), whereby $\HF_V(n) = n + c_V$ for all sufficiently large $n$. However, the constant $c_V$ grows in the number of additional points in $V$. 
Clearly some control on this effect is necessary since if the number of unioned points is sufficiently large, then the dimensionality of the coordinate ring increases.

\paragraph*{Computing properties of a variety.} A \emph{monomial ideal} is an ideal that is generated by monomials. If $I(V)$ is a monomial ideal, then the Hilbert function $\HF_{V}(n)$ is equal to the number of monomials in $\R_{\leq n}[x]$ that are not in $I$. These monomials are called \emph{standard monomials} and can be counted via inclusion-exclusion given a list of monomial generators for $I$. The dimension and degree of $V$ can then be read off from the Hilbert function via Lemma~\ref{lem:hilbert}.
\par Of course, $I(V)$ is not always a monomial ideal. For general varieties $V$, the Hilbert function can be computed algorithmically using Gr\"obner bases. Fix a graded monomial ordering on $\R[x]$.
The \emph{leading term} of a polynomial is the largest monomial w.r.t. that ordering.
The \emph{leading term ideal} $\LT(I)$ of $I$ is the monomial ideal generated by the leading term of each element of $I$. A \emph{Gr\"obner basis} for $I$ is a finite set of generators $\{g_i\}$ such that $\LT(I)$ equals the ideal generated by $\{\LT(g_i)\}$. The reason that a Gr\"obner basis helps to compute Hilbert functions is the following lemma, which reduces the general case of arbitrary ideals $I$ to the simpler case of monomial ideals.

\begin{lemma}[Reduction from general ideals to monomial ideals]\label{lem:LT}
	Let $I$ be an ideal in $\R[x]$. For any graded monomial ordering on $\R[x]$, the Hilbert functions of $V(I)$ and $V(\LT(I))$ are identical.
\end{lemma}

Thus, given a Gr\"obner basis $\{g_i\}$ of an arbitrary ideal $I$, one can form a description of the leading term ideal $\LT(I)$ as the monomial ideal generated by $\{\LT(g_i)\}$, and then use this to compute the Hilbert function of $V(I)$ by counting the number of standard monomials in $\LT(I)$. This machinery is demonstrated through several concrete examples when we use it in \S\ref{sec:high-prec} and \S\ref{sec:ex}.

\subsection{Standard approaches for kernel approximation}\label{ssec:prelim:standard}

Here we describe two of the most popular low-rank approximation approaches since we build upon them in the sequel. These are polynomial-based approaches~\citep{YanDurDav05,CotKesSre11,WanLiDar18,zwicknagl2009power, schaback2008limit, wathen2015spectral} and Random Fourier Features (RFF)~\citep{RahRec08,liu2020random}. The former is suited for high-precision regimes, whereas the latter is suited for high-dimensional regimes.
Each of these two approaches consists of two steps:
\begin{enumerate}
	\item Expand the relevant kernel
	\begin{align}
		K(x,y) = \int f_{\omega}(x) f_{\omega}(y) d\nu(\omega) \label{eq:prelim:int-rep}
	\end{align}
	as a convex combination of rank-$1$ functions. Here $\nu$ is a probability distribution---continuous if the representation~\eqref{eq:prelim:int-rep} is an integral, or discrete if~\eqref{eq:prelim:int-rep} is a sum.
	\item Form a rank-$r$ approximation by taking $r$ of these infinitely many rank-$1$ functions. 
\end{enumerate}

\par To explain the difference between polynomial-based approaches and RFF, let us begin with how they perform step (2). On one hand, polynomial-based approaches greedily choose the $r$ rank-$1$ functions $f_{\omega}(x)f_{\omega}(y)$ with largest weight $d\nu(\omega)$. On the other hand, RFF independently samples $r$ random rank-$1$ functions $f_{\omega}(x)f_{\omega}(y)$ according to the distribution $d\nu(\omega)$. 

\par This difference in step (2) necessitates strikingly different kinds of representations~\eqref{eq:prelim:int-rep} in step (1). Intuitively, the greedy truncation scheme performs well on representations in which $\nu$ is a discrete distribution with rapidly decaying tails. In contrast, the random sampling scheme performs well on representations in which the functions $f_{\omega}$ have small magnitude. In particular, the representations~\eqref{eq:prelim:int-rep} used in step (1) are as follows.
\begin{itemize}
	\item \underline{Polynomial-based representation}:
	\begin{align}
		K(x,y)
		=
		\sum_{n=0}^{\infty} c_n p_n(x,y)
		=
		\sum_{n=0}^{\infty} c_n \sum_{\alpha=1}^{\binom{n+d-1}{n}} u_{n,\alpha}(x)v_{n,\alpha}(y).
		\label{eq:tay-rep}
	\end{align}
	This representation first expands $K$ into the sum of polynomial kernels $p_n$ of degree at most $n$ (typically via monomial expansions or Chebyshev expansions), and then factorizes each $p_n$. The inner sum is over $\binom{n+d-1}{n}$ polynomials because this is the dimension of the space of degree-$n$ homogeneous polynomials on $\R^d$.
	\item \underline{RFF representation}:
	\begin{align}
		K(x,y) = \E_{\omega,\theta}
		\left[ \cos(\langle \omega, x \rangle + \theta) \cos(\langle \omega, y \rangle + \theta)  \right].	
		\label{eq:rff-int-rep}
	\end{align}
	Here $\omega$ is sampled from the Fourier transform $\mu$ of $k(x-y) = K(x,y)$; this is a probability distribution by Bochner's Theorem if $\mu$ is a continuous, PD, translation-invariant kernel normalized so that $K(0,0)=1$, see e.g.,~\citep{WilRas06}. Independently, $\theta$ is sampled from the uniform distribution over $[0,2\pi)$. This representation is obtained by simple trigonometric manipulation of the Fourier transform identity. 
\end{itemize}

Given that both approaches seek to optimize the $L^{\infty}$ error metric, a natural question is why use one representation and not the other? The answer is based on the parameter regime. 
\par On one hand, the fact that the representation~\eqref{eq:tay-rep} is a finite sum with coefficients that decay exponentially fast if $K$ is smooth, means that exponentially small error is obtained by truncating. This is critical for high-precision regimes. However, the issue with this approach is that the rank grows as $\Omega(n)^d$ in the truncation degree $n$, and this prohibitive beyond low dimensions $d$.
\par On the other hand, the fact that the integrand in the RFF representation~\eqref{eq:rff-int-rep} is bounded in magnitude by $1$, means that sampling-based quadrature converges at standard statistical rates which scale well in the dimension $d$. This is critical for high-dimensional regimes. However, the issue with this approach is that statistical rates require roughly $\Omega(1/\eps^2)$ samples in order to obtain $\eps$ accuracy, which is prohibitive for accuracies $\eps$ beyond a few digits of precision.
\par Details on each of these methods and their formal guarantees follow.

\subsubsection{High-precision approximation via polynomial features}
As discussed in \S\ref{ssec:intro:cont}, there are many ways to quantify smoothness of a kernel. For simplicity, we assume (1) isotropy, meaning that the kernel admits a representation $K(x,y) = f(\|x-y\|^2)$; and (2) the univariate function $f$ satisfies the following smoothness condition. Note that we write $K(x,y) = f(\|x-y\|^2)$ rather than $f(\|x-y\|)$ since in the latter case, $K$ might not be smooth even if $f$ is; e.g., take $f$ to be the identity.

\begin{assumption}\label{assump:smooth}
	There exist constants $\alpha > 0$ and $\beta \in (0,1)$ such that for all $n \in \N$, there is a polynomial $p_n$ of degree $n$ satisfying $\|f - p_n\|_{[0,4]} \leq \alpha \cdot \beta^{n}$.
\end{assumption}

A classical result of Bernstein from over a century ago shows that this assumption is essentially equivalent to analyticity of $f$ in a complex neighborhood around the approximation domain.
See also~\citep{Tre13} for other smoothness conditions that lead to fast rates for polynomial approximation.

\begin{lemma}[Analyticity implies approximation~\citep{bernstein1912ordre}]\label{lem:smooth}
	Suppose $f$ is analytically continuable to the Bernstein ellipse $E_{\rho} = \mathrm{Interior}(\{z + z^{-1} + 2 : z \in \C, |z| = \rho\})$ of parameter $\rho > 1$ around $[0,4]$. Then $f$ satisfies Assumption~\ref{assump:smooth} with $\alpha = 2 \|f\|_{E_{\rho}}/(\rho-1)$  and $\beta = 1/\rho$. 
\end{lemma}

This lemma ensures that Assumption~\ref{assump:smooth} is satisfied for popular kernels such as the Gaussian kernel, in which case $f(t) = e^{-t/2}$ is entire, and the Cauchy kernel, in which case $f(t) = (1+t/2)^{-1}$ has a pole at $-3/2$ and thus is analytically continuable to any Bernstein ellipse $E_{\rho}$ with parameter $\rho < \rho_{\max} = 7/4 + \sqrt{33}/4 \approx 3.186$.

Standard rates for approximating smooth isotropic kernels are immediate from combining this implication of smoothness with simple rank bounds on bounded-degree polynomial kernels, see e.g.,~\citep{WanLiDar18}. For completeness, we provide a short proof.

\begin{prop}[Standard rates for high-precision approximation]\label{prop:poly-standard}
	Suppose $K(x,y) = f(\|x-y\|^2)$ where $f$ satisfies Assumption~\ref{assump:smooth}. There is a universal constant $c$ such that for all $\eps > 0$, there is a kernel of rank 
	\[
	r \leq 
	\left(c \frac{\log(\alpha/\eps)}{\log (1/\beta)} \right)^{d}
	= O\left(\log1/\eps \right)^d.
	\] 
	that approximates $K$ on $\ball \times \ball$ to $L^{\infty}$ error $\eps$.
\end{prop}
\begin{proof}
	By Assumption~\ref{assump:smooth}, there is a polynomial $p_n$ of degree $n = \lceil \log(\alpha/\eps)/\log(1/\beta) \rceil$ satisfying $\|f-p_n\|_{[0,4]} \leq \eps$. Thus the kernel $K_n(x,y) = p_n(\|x-y\|^2)$ satisfies $\|K - K_n\|_{\ball \times \ball} \leq \eps$. Since $K_n$ has degree at most $2n$ in each argument, it follows from Lemma~\ref{lem:poly-dim} and a crude bound that $\rank K_n \leq \dim(\R_{\leq 2n}[x]) = \binom{2n+d}{d} = O(n)^d$.
\end{proof}

\begin{remark}[Taylor Features]\label{rem:tay}
	For the Gaussian kernel $G(x,y) = e^{-\|x-y\|^2/(2\sig^2)}$, one can obtain rank bounds which are slightly better in practice albeit the same asymptotically. The trick is to factor out the scalings $e^{-(\|x\|^2+\|y\|^2)/(2\sig^2)}$ and then approximate the remainder $e^{\langle x,y \rangle/\sig^2}$ via a rotation-invariant polynomial kernel $p_n(\langle x, y\rangle)$. The point is that the scaling factors do not affect the rank, and a rotation-invariant polynomial kernel $p_n(\langle x,y \rangle)$ generically has lower rank than an isotropic polynomial kernel $q_n(\|x-y\|^2)$ for $p_n$ and $q_n$ of the same degree (although both ranks are asympotically the same $O(n)^d$). Specifically, the popular Taylor Features kernel is
	$
		T_n(x,y) = e^{-(\|x\|^2+\|y\|^2)/(2\sig^2)} \sum_{k=0}^n \tfrac{\langle x,y \rangle^k}{\sig^{2k} k!}$, see e.g.,~\citep{YanDurDav05,YanDurGum03,CotKesSre11}.
\end{remark}

\subsubsection{High-dimensional approximation via Random Fourier Features}

The RFF kernel of rank $r$ is the empirical mean of $r$ samples of the integral representation~\eqref{eq:rff-int-rep}; that is,
\begin{align}
	K_r(x,y) := \frac{1}{r} \sum_{i=1}^r \cos(\langle \omega_i, x_i \rangle + \theta_i) \cos(\langle \omega_i, y \rangle + \theta_i)
	\label{eq:rff}
\end{align}
where $\omega_1, \dots, \omega_r$ are sampled from the Fourier transform $\mu$ of $k(x-y) = K(x,y)$, and $\theta_1, \dots, \theta_r \sim \Unif([0,2\pi))$ are all sampled independently. The guarantees of this approach are summarized as follows; see~\citep{RahRec08} for a proof. 

\begin{assumption}\label{assump:rff}
	$K$ is a continuous, positive-definite, translation-invariant kernel on $\R^d$ with normalization $K(0,0)=1$ and curvature $\sig_K^2$. 
\end{assumption}

\begin{prop}[Standard rates for high-dimensional approximation]\label{prop:rff-standard}
	Suppose $K$ satisfies Assumption~\ref{assump:rff}. Then the kernel $K_r$ in~\eqref{eq:rff} has rank at most $r$ and satisfies $\| K - K_r \|_{\ball \times \ball} \leq \eps$ with any constant probability for 
	\[
	r = O\left( d \,\frac{\log(\sig_K / \eps)}{\eps^2}  \right).
	\]
\end{prop}

Note that for a kernel $K$ satisfying Assumption~\ref{assump:rff}, its curvature $\sig_K^2$, defined as the trace of the Hessian of $k(x-y) = K(x,y)$ at $0$, is equal to $\E_{\omega \sim \mu} \|\omega\|^2$~\citep{RahRec08}.

\section{Kernel approximation over a variety: high-precision regime}\label{sec:high-prec}

Here we provide an exponential improvement in the rank~\eqref{eq:curse-of-dim} required by high-precision approaches for kernel approximation. Specifically, we show that if the approximation domain is a low-dimensional algebraic variety $V$ in a high-dimensional ambient space $\R^d$, then the curse of dimensionality for high-precision approaches can be alleviated: the exponential dependence in the ambient dimension $d$ is improvable to exponential dependence in the variety's dimension $\dim V$.

\begin{theorem}[High-precision approximation over a variety]\label{thm:main-var} 
	Suppose $K(x,y) = f(\|x-y\|^2)$ where $f$ satisfies Assumption~\ref{assump:smooth}. Suppose also $\cX = V \cap \ball$, where $V \subset \R^d$ is an equidimensional real algebraic variety. 
	There is a universal constant $c$ such that for all $\eps > 0$, there is a kernel of rank 
	\begin{align}
		r \leq \deg (V) \left(c \frac{\log(\alpha/\eps)}{\log (1/\beta)} \right)^{\dim V}.
		\label{eq:main-var:r}
	\end{align}
	that approximates $K$ on $\cX \times \cX$ to $L^{\infty}$ error $\eps$.
\end{theorem}

\par As overviewed in \S\ref{ssec:intro:tech}, our approach has two components. The first controls the approximation error and is standard: exploit smoothness in order to approximate the kernel by a low-degree polynomial. The second controls the rank of our approximate kernel and is the critical new ingredient: exploit the algebraic structure of the domain in order to factorize the polynomial kernel in a succinct way. A simple statement of this second ingredient that gives asymptotic bounds is as follows; this rank bound is generically tight (see Remark~\ref{rem:rank-asymp:tight}).

\begin{prop}[Rank bound for polynomial kernels over varieties]\label{prop:rank-asymp}
	Let $K_n$ be the restriction of an (indefinite) degree-$n$ polynomial kernel
	 to $V \times V$, where $V$ is a variety in $\R^d$. Then
	\begin{align}
		\rank K_n \leq
				\HF_{V}(n).
		\label{eq:rank-comp-HF}
	\end{align}
\end{prop}

\par With this rank bound, the proof of Theorem~\ref{thm:main-var} follows readily.

\begin{proof}[Proof of Theorem~\ref{thm:main-var}]
	By Assumption~\ref{assump:smooth}, there is a polynomial $p_n$ of degree $n = \lceil \log(\alpha/\eps)/\log(1/\beta) \rceil$ satisfying $\|f-p_n\|_{[0,4]} \leq \eps$. Thus the kernel 
	$K_{n}(x,y) = p_n(\|x-y\|^2)$ satisfies 
	$
		\|K - K_{n}\|_{\cX \times \cX} \leq \eps.
	$
	By Proposition~\ref{prop:rank-asymp}, the fact that $p_n(\|x-y\|^2)$ is a polynomial of degree\footnote{Although this degree increase for $n$ to $2n$ is irrelevant for the asymptotics in Theorem~\ref{thm:main-var}, a more refined analysis of isotropic polynomial kernels yields rank bounds that are better in practice, see Remark~\ref{rem:rank-asymp:use}.} $2n$, and Lemma~\ref{lem:hf-equidim},
	$
		\rank K_{n} \leq \deg V \binom{2n+\dim V}{\dim V}$. This is at most $\deg V \left( c n \right)^{\dim V}$
	for some universal constant $c$.
\end{proof}

The remainder of the section is devoted to proving Proposition~\ref{prop:rank-asymp}. Along the way, we develop a more general understanding of how the rank of a polynomial kernel drops when it is restricted to an algebraic variety, since this may be of independent interest. In particular, we provide several illustrative examples in \S\ref{sssec:rank-ex}, describe the correspondence between polynomials over varieties and bilinear forms over coordinate rings in \S\ref{sssec:rank-forms}, provide an exact rank formula in terms of a finite-dimensional matrix in \S\ref{sssec:rank-exact}, and prove the asymptotic rank bound in Proposition~\ref{prop:rank-asymp} as well as remark on its tightness and common use cases in \S\ref{sssec:rank-asymp}.

\subsection{Rank of polynomial kernels over algebraic varieties}\label{ssec:hp-rank}

\subsubsection{Illustrative examples}\label{sssec:rank-ex}

\par We begin by illustrating the underlying phenomenon through several simple examples (see \S\ref{sec:ex} for examples with more involved varieties). For simplicity, we consider the rotation-invariant kernel
\[
	R_n(x,y) = \sum_{k=0}^n \frac{\langle x,y \rangle^k}{k!}.
\]
The same ideas extend to isotropic kernels, see Remark~\ref{rem:rank-asymp:use}. The significance of this kernel $R_n(x,y)$ is that it is the ``Taylor Features'' approximation of the Gaussian kernel $e^{-\|x-y\|^2/2}$, see Remark~\ref{rem:tay}, modulo omitting the scalings $e^{-(\|x\|+\|y\|^2)/2}$ which does not change the rank. In what follows, we abuse notation slightly by writing $R_n(V)$ to denote the restriction of the kernel $R_n$ to $V \times V$.

\begin{example}\label{ex:toy}
	We demonstrate that the rank of $R_2$ drops when restricted to a variety $V \subset \R^2$.
	\begin{itemize}
		\item \underline{Full space.} If $V = \R^2$, then $R_2(x,y) = \langle \phi(x),\phi(y)\rangle$ where $\phi(x) = [1, x_1, x_2, x_1^2/\sqrt{2},x_2^2/\sqrt{2}, x_1x_2]^T \in \R^6$. Thus $\rank R_2(V) \leq 6$.
		\item \underline{1-sparse data.} If $V = \{x \in \R^2 : x_1x_2 = 0\}$, then $R_2(x,y) = \langle \phi(x),\phi(y)\rangle$ where $\phi(x) = [1, x_1, x_2, x_1^2/\sqrt{2},x_2^2/\sqrt{2}]^T \in \R^5$. Thus $\rank R_2(V) \leq 5$.
		\item \underline{Spherical data.} If $V = \{x \in \R^2 : x_1^2 + x_2^2 = 1\}$, then $R_2(x,y) = \langle \phi(x),\phi(y) \rangle$ where $\phi(x) = [	\sqrt{5}/2, x_1, x_2,  x_1^2 -1/2, x_1x_2]^T \in \R^5$. Thus $\rank R_2(V) \leq 5$.
	\end{itemize}
	Note that the rank bounds in all these examples are tight, as shown next.
\end{example}

\par While the computations are straightforward in this toy example, computing rank bounds is clearly much more involved for more complicated kernels and varieties. Proposition~\ref{prop:rank-asymp} provides a simple, systematic approach for computing tight rank bounds.

\begin{example}[Using Proposition~\ref{prop:rank-asymp}]\label{ex:ip-bound}
	Let us demonstrate how to use Proposition~\ref{prop:rank-asymp} to compute the rank of $R_n$ when restricted to a variety. Since Proposition~\ref{prop:rank-asymp} is tight for the Taylor Features kernel (Remark~\ref{rem:rank-asymp:tight}), $\rank R_n(V)
		 = \HF_{V}(n)
	 $ for any variety $V \subset \R^d$.
	\begin{itemize}
		\item \underline{Full space.} If $V = \R^d$, then by Lemma~\ref{lem:poly-dim},
		\[
		\rank R_n(V) = \HF_{V}(n) = \binom{n+d}{d}.
		\] 
		\item \underline{$1$-sparse data.} If $V = \{x \in \R^d : x_ix_j = 0, \forall i < j \}$, then $I(V) = \langle x_ix_j : i < j \rangle$ is a monomial ideal and the standard monomials of degree at most $n$ are $1$ and $\{x_i^k\}_{i \in [d], k \in [n]}$.
		Thus
		\[
		\rank R_n(V) = \HF_{V}(n) = nd+1.
		\] 
		\item \underline{Spherical data.} If $V = \{x \in \R^d : \|x\|^2 = 1 \}$, then $I(V)$ is not a monomial ideal. The polynomial $\sum_{i=1}^d x_i^2 - 1$ generates $I(V)$ and forms a Gr\"obner basis for it w.r.t. grevlex, say. Thus $\LT(I(V)) = \langle x_1^2 \rangle$. 
		The corresponding standard monomials of degree at most $n$ are (i) monomials in $\R_{\leq n}[x_2, \dots, x_d]$; and (ii) $x_1$ times monomials in $\R_{\leq n-1}[x_2, \dots, x_d]$.
		Thus
		\[
		\rank R_n(V) = \HF_{V}(n) = \binom{n+d-1}{d-1} + \binom{n+d-2}{d-1}.
		\]
	\end{itemize}
	This proves optimality of the bounds in Examples~\ref{ex:toy} for $n=d=2$. Moreover, it shows that $\rank R_n(V)$ grows as $O(n)^d$, $O(n)$, and $O(n)^{d-1}$, respectively, because $\dim(V)$ is $d$, $1$, and $d-1$ for these three varieties $V$.
\end{example}

\subsubsection{Polynomial kernels over varieties as bilinear forms over coordinate rings}\label{sssec:rank-forms}

\par Our starting point for developing rank bounds is to view polynomial kernels as symmetric bilinear forms. First consider a polynomial kernel $K$ on the full space $\R^d \times \R^d$. 
Recall that the degree of $K$ is the total degree in either variable $x, y \in \R^d$ (these two numbers are the same by symmetry).  
A basic fact is that polynomials kernels of degree at most $n$ are in
$1$-$1$ correspondence with symmetric bilinear forms $C$ over $\R_{\leq n}[x]$. This fact is perhaps most intuitively understood in the monomial basis, where it amounts to the identity
\begin{align}
	K(x,y) = f(x)^T C f(y), \label{eq:sbf}
\end{align}
where $f(x)$ has entries $x^{\alpha} = \prod_{i=1}^d x_i^{\alpha_i}$ and $C$ has entries $c_{\alpha,\beta}$ for multi-indices $\alpha,\beta \in \N_0^d$ satisfying $\sum_{i=1}^d \alpha_i$, $\sum_{i=1}^d \beta_i\leq n$, and $K(x,y) = \sum_{\alpha,\beta} c_{\alpha,\beta} x^{\alpha}y^{\beta}$.

\par How does this change if $K$ is restricted to $V \times V$, where $V$ is a variety in $\R^d$? This corresponds to restricting the symmetric bilinear form~\eqref{eq:sbf} to the space $\R_{\leq n}[V]$ of bounded-degree polynomials in the coordinate ring. (Recall from \S\ref{ssec:prelim:ag} that $\R_{\leq n}[V] = \R_{\leq n}[x] / I_{\leq n}(V)$.) This restricted form over $\R_{\leq n}[V]$ is easily computed in terms of the unrestricted form over $\R_{\leq n}[x]$ and the linear restriction map $\Phi : \R_{\leq n}[x] \to \R_{\leq n}[V]$, which maps a polynomial over $\R^d$ to its restriction over $V$. Note that $\Phi^T$ corresponds to the map that embeds the coordinate ring into the polynomial ring.

\begin{lemma}[Polynomial kernels over varieties as symmetric bilinear forms over coordinate rings]\label{lem:ker-sbf}
	Let $K$ be the restriction of an (indefinite) degree-$n$ polynomial to $V \times V$, where $V$ is a variety in $\R^d$. 
	Then $K$ is equal to the symmetric bilinear form
	\begin{align*}
		\Phi C \Phi^T
	\end{align*}
	over $\R_{\leq n}[V]$, where $\Phi$ is the restriction map from $\R_{\leq n}[x]$ to $\R_{\leq n}[V]$, and $C$ is the symmetric bilinear form over $\R_{\leq n}[x]$ corresponding to $K$.
\end{lemma}
\begin{proof}
	Although the lemma statement is basis-free, the proof is perhaps most intuitive by choosing the following convenient bases. Since $\Phi$ is the restriction map from $\R_{\leq n}[x]$ to $\R_{\leq n}[V]$, there is a polynomial basis $f_1, \dots, f_N$ of $\R_{\leq n}[x]$ such that (the equivalence classes corresponding to) $f_1, \dots, f_M$ form a basis for $\R_{\leq n}[V]$, and moreover $\Phi f_i = f_i$ for $i \in [M]$ and $\Phi f_i = 0$ for $i \in [N] \setminus [M]$. (Such a basis can be computed e.g., using Gr\"obner bases.) Abusing notation slightly, let $C$ and $\Phi$ denote the matrices corresponding to the respective linear maps w.r.t. these bases. Then $\Phi = [I_{M \times M}, 0_{M \times (N-M)}]$, and
	\begin{align}
		K(x,y)
		= f(x)^T C f(y)
		= g(x)^T \Phi C\Phi^T g(y),
		\qquad \forall x,y \in V,
		\label{eq:rank-exact:1}
	\end{align}
	where $g(x) := [f_1(x), \dots, f_M(x)]^T$.
\end{proof}

\subsubsection{Exact rank formula}\label{sssec:rank-exact}

The correspondence between polynomial kernels over varieties and symmetric bilinear forms over coordinate rings in Lemma~\ref{lem:ker-sbf} gives an exact formula for the rank of the former in terms of the rank of a finite-dimensional matrix. 

\begin{prop}[Exact rank formula for polynomial kernels over varieties]\label{prop:rank-exact}
	Consider the setup in Lemma~\ref{lem:ker-sbf}. Then the rank of $K$ over $V \times V$ is
	\begin{align}
		\rank K = 
		\rank \Phi C \Phi^T.
		\label{eq:rank-exact}
	\end{align}
	In particular, if $C$ is positive definite, then
		\begin{align}
			\rank K = \rank \Phi = \HF_{V}(n).
			\label{eq:rank-exact-simp}
		\end{align}

\end{prop}

The proof makes use of the following generalization of the ``Unisolvence Theorem'' from the standard setting of $\R$ to the present setting of real algebraic varieties in $\R^d$.
For convenience, we state this in terms of the invertibility of a generalized Vandermonde matrix.

\begin{lemma}[Unisolvence Theorem on varieties]\label{lem:rank-exact:interp}
	Suppose $V$ is a variety in $\R^d$. Let $M$ denote $\dim(\R_{\leq n}[V])$. There exist points $x_1, \dots, x_M \in V$ such that the matrix $S \in \R^{M \times M}$ with entries $S_{ij} = f_i(x_j)$, is non-singular for any basis $f_1, \dots, f_M$ of $\R_{\leq n}[V]$.
\end{lemma}
\begin{proof}
	It suffices to show the claim for any fixed basis $f_1, \dots, f_M$. Define $f(x) := [f_1(x),\dots,f_M(x)]^T$, and let $W := \textrm{span}(\{f(x)\}_{x \in V}) \subseteq \R^M$. Since $W$ is a finite-dimensional vector space, it admits a basis of the form $\{f(x_1), \dots, f(x_k)\}$ for some $x_1, \dots, x_k \in V$. Clearly $k \leq M$. Assume for contradiction that $k < M$; else the claim follows. We make two observations. First, the $M \times k$ matrix with $ij$-th entry $f_i(x_j)$ has deficient row rank, thus there exist $\alpha_1, \dots, \alpha_M$ not all zero such that
	\[
	\sum_{i=1}^M \alpha_i f_i(x_j) = 0, \qquad \forall j \in [k].
	\] 
	Second, since $\{f(x_1), \dots, f(x_k)\}$ is a basis of $W$, there exist functions $\ell_1, \dots, \ell_k : V \to \R$ satisfying
	\[
		f(x) = \sum_{j=1}^k \ell_j(x) f(x_j), \qquad \forall x \in V.
	\]
	(These are the corresponding Lagrange interpolating polynomials.) From these two observations it follows that $\sum_{i=1}^M \alpha_i f_i \equiv 0$ on $V$. Indeed, for all $x \in V$, 
	\[
		\sum_{i=1}^M \alpha_i f_i(x) = \sum_{j=1}^k \ell_j(x) \sum_{i=1}^M \alpha_i f_i(x_j) = 0.
	\]
	This contradicts the fact that $\{f_1,\dots,f_M\}$ is a basis of of $\R_{\leq n}[V]$.
\end{proof}

\begin{proof}[Proof of Proposition~\ref{prop:rank-exact}]
	Consider the basis choice in the proof of Lemma~\ref{lem:ker-sbf}.
	\par \underline{Proof of ``$\leq$''.} 
	Let $A^TB$ be a factorization of $\Phi C\Phi^T$ where $A$ and $B$ have $\rank \Phi C \Phi^T$ rows. Denote $\phi(x) = Ag(x)$ and $\psi(y) = Bg(y)$. Then by~\eqref{eq:rank-exact:1}, 
	\begin{align}
		K(x,y) = 
		\langle \phi(x), \psi(y) \rangle
		\label{eq:rank-exact:fact}
	\end{align}
	is an explicit factorization of $K$ over $V \times V$ of rank equal to $\rank \Phi C \Phi^T$.
	
	\par \underline{Proof of ``$\geq$''.} Consider points $x_1, \dots, x_M \in V$ guaranteed by Lemma~\ref{lem:rank-exact:interp},
	and let $H \in \R^{M \times M}$ be the corresponding kernel matrix with entries $H_{ij} = K(x_i,x_j)$. Note that 
	\[
	\rank K \geq \rank H,
	\]
	since the rank of the kernel $K$ over $V \times V$ is at least the rank of the kernel $K$ restricted to $\{x_1, \dots, x_M\} \times \{x_1, \dots, x_M\} \subseteq V \times V$, which in turn is precisely the rank of the matrix $H$. Now to bound $\rank H$, use~\eqref{eq:rank-exact:1} to write $H = S^T \Phi C \Phi^T S$ where $S \in \R^{M \times M}$ is the generalized Vandermonde matrix in Lemma~\ref{lem:rank-exact:interp} with entries $S_{ij} = f_i(x_j)$. Since $S$ is invertible,
	\[
	\rank H = \rank  S^T \Psi C \Psi^T S = \rank \Phi C \Phi^T.
	\]
	\par \underline{Corollary when $C$ is positive definite.} In this case, $C = LL^T$ for an invertible matrix $L$. Thus
		\[
			\rank \Phi C \Phi^T = \rank (\Phi L)(\Phi L )^T = \rank \Phi L = \rank \Phi.
		\]
		Since $\Phi$ is the linear projection map onto $\R_{\leq n}[V]$, 
		the rank of $\Phi$ is the dimension of this space---which is by definition the Hilbert function $\HF_{V}(n)$.
\end{proof}

Note that this proof does more than establish the rank of $K$. It also identifies the feature space (the image of $\Phi C \Phi^T$ viewed as a subspace of $\R_{\leq n}[V]$), and an optimal factorization~\eqref{eq:rank-exact:fact}. Since this factorization has polynomial entries, we obtain the following corollary. 

\begin{cor}[Polynomial kernels over varieties have optimal polynomial factorizations]\label{cor:poly-factor}
	Consider the setup in Lemma~\ref{lem:ker-sbf}. There exist polynomial functions $\phi, \psi : V \to \R^{\rank K}$ such that $K(x,y) = \langle \phi(x), \psi(y) \rangle$ for all $x,y \in V$. Moreover, if $K$ is PD, then this holds with $\phi = \psi$.
\end{cor}

\subsubsection{Asymptotic rank formula}\label{sssec:rank-asymp}

While Proposition~\ref{prop:rank-exact} provides an exact formula for the rank of an arbitrary polynomial kernel over an arbitrary variety, it involves a matrix that is large even for moderate degree $n$ and dimension $d$. Proposition~\ref{prop:rank-asymp} provides a bound whose computation does not involve large matrices. The price to pay is that this bound is oblivious to the structure of $p$ beyond its degree. Nevertheless, this bound is tight for generic kernels. We now show how Proposition~\ref{prop:rank-asymp} follows from Proposition~\ref{prop:rank-exact}.

\begin{proof}[Proof of Proposition~\ref{prop:rank-asymp}]
	By Proposition~\ref{prop:rank-exact}, a dimension bound, and the definition of the Hilbert function,
	$\rank K = \rank(\Phi C \Phi^T) \leq \rank \Phi = \dim(\R_{\leq n}[V])
	= \HF_{V}(n)$.
\end{proof}

We conclude with two remarks about Proposition~\ref{prop:rank-asymp}: tightness and common use cases.

\begin{remark}[Proposition~\ref{prop:rank-asymp} is generically tight]\label{rem:rank-asymp:tight}
	The rank bound~\eqref{eq:rank-comp-HF} is exact for ``generic'' polynomial kernels. Indeed, the only inequality in the proof was $\rank \Phi C \Phi^T \leq \rank \Phi$, and this holds with equality by Proposition~\ref{prop:rank-exact} if the bilinear form $C$ is PD (e.g., for Taylor Features\footnote{Ignoring the exponential scalings which do not affect the rank, see Remark~\ref{rem:tay}.}).
\end{remark}

\begin{remark}[Proposition~\ref{prop:rank-asymp} for common polynomial kernels]\label{rem:rank-asymp:use}
	Polynomial kernels of interest are typically rotation-invariant or isotropic; that is, of the form $R_n(x,y) = p_n(\langle x,y\rangle)$ or $K_n(x,y) = p_n(\|x-y\|^2)$, respectively, where $p_n$ is a univariate polynomial of degree $n$. Abuse notation slightly to denote the restrictions of these kernels to $V \times V$ by $R_n(V)$ and $K_n(V)$, respectively.
	\begin{itemize}
		\item \underline{Rotation-invariant kernels.} 
		In this case, Proposition~\ref{prop:rank-asymp} combined with the asymptotics in Lemma~\ref{lem:hilbert} gives the tight bound
		\[
			\rank R_n(V) \leq \HF_{V}(n) = \deg(V) n^{\dim V} + O(n^{\dim V - 1}) 
		\]
		which holds with equality for generic polynomials $p_n$ (and e.g., Taylor Features).
		\item \underline{Isotropic kernels.} 
		In this case, a direct application of Proposition~\ref{prop:rank-asymp} gives a loose bound since it treats $p_n(\|x-y\|^2)$ as a generic polynomial of degree $2n$, leading to asymptotics of order
		\[
			\rank K_n(V) \leq \HF_{V}(2n) = \deg(V) (2n)^{\dim V} + O(n^{\dim V - 1}).
		\] 
		A more refined analysis can essentially improve the $2n$ to $n$ by capturing the structure of the polynomial $p_n(\|x-y\|^2)$ beyond its degree.
		This has an important effect in practice. However, it does not change the asymptotic bounds for kernel approximation (c.f., Theorem~\ref{thm:main-var}) since $n$ is only specified up to a constant anyways. As such, we do not investigate this further here.
	\end{itemize}
	
\end{remark}

\section{Kernel approximation over a variety: high-dimensional regime}\label{sec:hd}

Here we show that high-dimensional approaches for kernel approximation perform substantially better if the approximation domain is a low-dimensional algebraic variety $V$ in a high-dimensional ambient space $\R^d$. Specifically, we improve the rank bound~\eqref{eq:rate-high-dim} of standard approaches by showing that their dependence on the ambient dimension $d$ can be improved to dependence on $\dim V$.

\begin{theorem}[High-dimensional approximation over a variety]\label{thm:high-dim}
	Suppose $K$ is a kernel satisfying Assumption~\ref{assump:rff}. Suppose also $\cX = V \cap \ball$, where $V \subset \R^d$ is an equidimensional real algebraic variety satisfying $\deg V \leq e^{\dim V}$. For any $\eps > 0$, there exists a kernel of rank 
	\begin{align}
		r = 
		O\left(  \frac{\dim V \, \log (\sig_K \dim V /\eps)}{\eps^2} 
		\right)
		\label{thm:high-dim-rank}
	\end{align}
	that approximates $K$ on $\cX \times \cX$ to $L^{\infty}$ error $\eps$.
\end{theorem}

We prove this existential result in an algorithmic way. We use the Random Fourier Features kernel $K_r$ defined in~\eqref{eq:rff}, except with one minor technical modification: rather than sampling frequencies $\omega_1, \dots, \omega_r$ from the Fourier distribution, we sample them from a truncated version of it. That is, we re-sample a frequency if its squared norm is large (roughly $\E\|\omega\|^2/\eps = \sig_K^2/\eps$). We prove that this construction works with high probability: if the rank is
\begin{align}
		r = O\left( \frac{\dim V \, \log (\sig_K \dim V  /\eps) + \logdel}{\eps^2}  \right),
		\label{thm:high-dim-rank-hp}
\end{align}
then this RFF kernel approximates $K$ on $\cX \times \cX$ to $L^{\infty}$ error $\eps$ with probability at least $\del$. This implies the existential result in Theorem~\ref{thm:high-dim} by taking $\delta$ to be any constant.

\begin{remark}[Degree]\label{rem:high-dim-deg}
	Dependence on $\deg V$ is unavoidable, see \S\ref{ssec:intro:cont}. In Theorem~\ref{thm:high-dim}, we assume for simplicity of presentation that $\deg V$ is not exponentially large, since in this case the contribution of $\deg V$ to the rank bound~\eqref{thm:high-dim-rank} is negligible. The proof extends to arbitrary degree essentially without change\footnote{
		The only difference is that the size $\log |\cX_n|$ of the norming sets in Proposition~\ref{prop:norming} increases by $\log \deg V + \dim V \log \log \deg V$
		in order to balance terms in the tensoring proof. This slight increase in $\log |\cX_n|$ results in an analogous increase in the final rank bound in Theorem~\ref{thm:high-dim-rank} since $r$ scales linearly in $\log |\cX_n|$, see~\eqref{eq:rff-pf:r}.}
	if the rank $r$ is increased by $\eps^{-2} (\log \deg V + \dim V \log \log \deg V)$.
\end{remark}

\begin{remark}[Logarithmic dependence]
	For the Gaussian kernel $G(x,y) = e^{-\|x-y\|^2/(2\sig^2)}$ and the Cauchy kernel $C(x,y) = (1+\|x\|^2/(2\sig^2))^{-1}$, Theorem~\ref{thm:high-dim} depends logarithmically on $d$ since $\log \sig_K = \log(d/\sig^2)$. By a similar proof technique (in particular using the key Proposition~\ref{prop:norming} about norming sets over varieties), we can show that an alternative approach based on polynomial approximation and then Johnson-Lindenstrauss projection achieves similar guarantees to Theorem~\ref{thm:high-dim}, with $\log d$ improved to $\log \dim V$. This generalizes to smooth isotropic kernels.
	However, we focus on RFF since its much better algorithmic efficiency outweighs this lower-order term.
\end{remark}

See \S\ref{ssec:intro:tech} for an overview of the proof of Theorem~\ref{thm:high-dim}. As explained there, the key ingredient beyond the standard RFF analysis is to exploit the rigidity of polynomials over varieties to ensure the existence of ``norming sets'' of small size. We develop this ingredient in \S\ref{ssec:hd-norming}, and then use it to prove Theorem~\ref{thm:high-dim} in \S\ref{ssec:hd-pf}.

\subsection{Constant norming sets for varieties}\label{ssec:hd-norming}

Given a compact set $\cX \subset \R^d$, slack $\lambda > 1$, and degree $n \in \N$, a \emph{norming set} is a subset $\cX_n \subseteq \cX$ satisfying $\|p\|_{\cX} \leq \lambda\|p\|_{\cX_n}$ for all polynomials $p$ of degree $n$. That norming sets of small cardinality exist is a classical result in approximation theory with many applications, for example bounding the operator norm of the interpolation projection operator~\citep{bos2018fekete}. 

\par In order to prove Theorem~\ref{thm:high-dim}, we require a version of this standard existential result that is strengthened in two important ways. One is that we require the slack $\lambda$ to be constant (say $2$), rather than the large $\binom{n+d}{d}$ that is sufficient for standard applications in interpolation theory. The other is that we require the cardinality of the norming set to not grow exponentially in the ambient dimension $d$; the standard bound is $|\cX_n|  = \binom{n+d}{d} = O(n)^d$. This requires exploiting the fact that in the setup of this paper, $\cX$ is a compact subset of an algebraic variety of dimension $\dim V \ll d$.

\begin{prop}[Constant norming set for variety]\label{prop:norming}
	Let $\cX$ be a compact subset of an equidimensional variety $V \subset \R^d$ satisfying $\deg V \leq e^{\dim V}$. For any integer $n \in \N$, there is a set $\cX_n \subset \cX$ of size $\log |\cX_n| = O( \dim V \log(n \dim V))$ satisfying $\|p\|_{\cX} \leq 2 \|p\|_{\cX_n}$ for all polynomials $p \in \R_{\leq n}[x]$.
\end{prop}

These two ``tightenings'' of the standard result are achieved by combining the classical argument (based on Fekete sets) with a tensoring trick inspired by Proposition 23 of~\citep{bloom2012polynomial}.
Let us first introduce this classical argument. Our exposition is based on~\citep{bos2018fekete}; see that nice survey for further background.

\par Fix a compact domain $\cX \subset \R^d$ and degree $n \in \N$. Let $N$ denote the dimension of the space $\R_{\leq n}[\cX]$ of polynomials of degree at most $n$ restricted to $\cX$, and let $\{p_1, \dots, p_N\}$
denote any basis of this space. 
Define $\mathrm{vdm}(x_1, \dots, x_N)$ to be the determinant of the $N \times N$ Vandermonde matrix with $ij$-th entry $p_i(x_j)$. A \emph{Fekete set} for $\cX$ of degree $n$ is a maximizer of
\[
	\max_{\cF_n \subset \cX, \, |\cF_n| = N} \abs{\mathrm{vdm}(\cF_n)}.
\]
(Note that this definition is independent of the choice of basis $\{p_1, \dots, p_N\}$.) A basic fact about Fekete sets is that they are norming sets, albeit of large size $|\cF_n| = N$ and for large slack $\lambda = N$. 

\begin{lemma}[Fekete sets are norming sets; folklore]\label{lem:fekete}
	Consider any compact domain $\cX \subset \R^d$ and degree $n \in \N$. 
	Set $N = \dim(\R_{\leq n}[\cX])$.
	Then any Fekete set $\cF_n$ satisfies $\|p\|_{\cX} \leq N \|p\|_{\cF_n}$. 
\end{lemma}
\begin{proof}
	This proof is folklore; we sketch it for completeness and refer to e.g.,~\citep{bos2018fekete} for details. Let $P$ denote the interpolation projection operator which given a continuous function $f$ over $\cX$, outputs a polynomial $P(f)$ of degree at most $n$ which interpolates $f$ at $\cF_n$. Then 
	\[
		P(f)(x) = \sum_{i=1}^N f(x_i) \ell_i(x),
	\]
	where $\ell_i(x)$ is the $i$-th Lagrange interpolating polynomial for $\cF_n$. Because of the classical identity
	$\ell_i(x) = \mathrm{vdm}(x_1, \dots, x_{i-1},x, x_{i+1}, \dots, x_N) / \mathrm{vdm}(x_1, \dots, x_{i-1},x_i, x_{i+1}, \dots, x_N)$, it follows by definition of $\cF_n$ being a Fekete set that
	\[
		\|\ell_i\|_{\cX} = 1.
	\]
	Thus, for any polynomial $p \in \R_{\leq n}[x]$, we have
	\[
		\|p\|_{\cX}
		=
		\left\|P(p)\right\|_{\cX}
		= 
		\left\|\sum_{i=1}^N p(x_i) \ell_i\right\|_{\cX}
		\leq
		\sum_{i=1}^N |p(x_i)| \cdot \left\|\ell_i\right\|_{\cX}
		\leq 
		N\|p\|_{\cF_n}.
	\]
\end{proof}

\begin{proof}[Proof of Proposition~\ref{prop:norming}]
	For shorthand, denote $\dim V$ by $d^{\star}$. For integer $a \in \N$ chosen shortly, let $\cX_n = \cF_{an}$ be a Fekete set for $\cX$ of degree $an$. Since $|\cF_{an}|$ is the dimension of the space $\R_{\leq an}[V]$,
	\[
	|\cX_n|
	=
	|\cF_{an}|
	=
	\HF_V(an) 
	\leq
	\deg V \binom{an + d^*}{d^*}
	\leq
	\deg V \cdot (2ean)^{d^{\star}}
	\]
	by Lemma~\ref{lem:hf-equidim} and a crude bound.
	By the assumption $\deg V \leq e^{d^{\star}}$, this implies
	\begin{align}
		|\cX_n| \leq (15 an)^{d^{\star}}.
		\label{eq:norming-proof}
	\end{align}
	Now observe that if 
	$\deg(p) \leq n$, then $\deg(p^{a}) \leq an$, thus
	\[
	\|p\|_{\cX} = \|p^a\|_{\cX}^{1/a} \leq \left( |\cF_{an}| \cdot \|p^a\|_{\cF_{an}} \right)^{1/a} = |\cF_{an}|^{1/a} \cdot \|p\|_{\cX_n},
	\]
	by the norming property of Fekete sets (Lemma~\ref{lem:fekete}). Choosing $a \geq 1$ strengthens the tensoring bound $|\cF_{an}|^{1/a}$. In particular, choosing $a = \lceil c d^* \log(n d^*) \rceil$ for an appropriate constant $c$ ensures that 
	$|\cF_{an}|^{1/a}
	\leq
	(15an)^{d^{\star}/a}
	=
	\exp( \frac{d^{\star}}{a}  \log (15an) )
	\leq 2$ 
	as well as the desired guarantee on $\log|\cX_n|$.
\end{proof}

\subsection{Proof of Theorem~\ref{thm:high-dim}}\label{ssec:hd-pf}

\paragraph*{Construction: RFF on truncated Fourier distribution.} By Bochner's Theorem,
\[
	K(x,y) = \E_{\omega \sim \mu} \left[ e^{i \langle \omega, x - y\rangle} \right].
\]
Define 
\begin{align}
	\tilde{K}(x,y) = \E_{\tilde{\omega} \sim \tilde{\mu}} \left[e^{i \langle \tilde{\omega}, x - y \rangle} \right],
	\label{eq:rff-int-rep-tilde}
\end{align}
where $\tilde{\mu}(\tilde{\omega}) = 	\frac{1}{1-p} \mu(\tilde{\omega}) \cdot \mathds{1}[\|\tilde{\omega}\|^2 \leq t]$
is $\mu$ truncated to the ball of squared norm $t = 2\E_{\omega \sim \mu} \|\omega\|^2 / \eps = 2\sig_K^2/\eps$. Here, $p = \Prob_{\omega \sim \mu} [\|\omega\|^2 > t]$ so that $\tilde{\mu}$ is rescaled to a probability distribution. 
\par The approximation we construct is the RFF approximation of $\tilde{K}$; that is,
\[
\tilde{K}_r(x,y) = \frac{1}{r}\sum_{i=1}^r e^{i \langle \omega, x - y\rangle},
\]
where $\tilde{\omega_1}, \dots, \tilde{\omega_r} \sim \tilde{\mu}$ are drawn independently. Note that $\tilde{K}$ has rank $2r$ by expanding the complex exponential into sinusoids~\citep{RahRec08}. We show that $\tilde{K}_r$ approximates $K$ well with high probability.

\paragraph*{Analysis step 1: Error bound for a fixed pair of points}
As in~\citep{RahRec08}, applying Hoeffding's inequality for sampling the integral representation~\eqref{eq:rff-int-rep-tilde} implies that for any fixed $x,y \in \cX$,
\begin{align}
	\Prob\left[ \abs{\tilde{K}(x,y) - {\tilde{K}}_r(x,y)} \leq \eps \right]
	\geq
	1 - e^{- r\eps^2/2 }.
	\label{eq:rff-pf:hoeffding}
\end{align}

\paragraph*{Analysis step 2: Extending the error bound to the whole domain} 
This is where the proof critically deviates from~\citep{RahRec08}: rather than union bound over an $\eps$-net of $\cX$, we union bound over a norming set for $\cX$ and exploit algebraic properties of the domain.

\underline{Fourier distribution truncation.} By Markov's inequality,
\[
p = \Prob_{\omega \sim \mu}\left[ \|\omega\|^2 > t  \right] \leq \frac{\E_{\omega \sim \mu}\|\omega\|^2}{t} = \frac{\eps}{2}.
\]
Thus the kernel $\tilde{K}$ is uniformly close to $K$ because for all $x,y$,
\begin{align}
	\abs{K(x,y) - \tilde{K}(x,y)}
	= 
	\abs{
		\E_{\omega \sim \mu} \left[ e^{i \langle \omega, x - y\rangle} \right]
		-
		E_{\tilde{\omega} \sim \tilde{\mu}} \left[ e^{i \langle \tilde{\omega}, x - y\rangle} \right]
	}
	\leq 2p
	\leq \eps.\label{eq:rff-pf:K-Ktilde}
\end{align}
Above, the second step is by conditioning on $\|w\|^2 \leq t$ and bounding the integrands by $1$.

\underline{Polynomial approximation.} We approximate $\tilde{K}$ and $\tilde{K_r}$ by low-degree polynomials. Since their Fourier Transforms are compactly supported, it can be shown (see Lemma~\ref{lem:ft-smooth}) that there exist polynomial kernels $P_n$ and $Q_n$ of degree $n = O(t + \log 1/\eps) =  O(\sig_K^2/\eps)$ that satisfy
\begin{align}
	\|\tilde{K} - P_n\|_{\ball \times \ball} \leq \eps \quad \text{ and } \quad 
	\|\tilde{K_r} - Q_n\|_{\ball \times \ball} \leq \eps.
	\label{eq:rff-pf:poly}
\end{align}

\underline{Using the norming set.} Let $\cX_n \subset \cX$ be the norming set guaranteed by Proposition~\ref{prop:norming}. Then by a union bound over~\eqref{eq:rff-pf:hoeffding} for all $(x,y) \in \cX_n \times \cX_n$, we have
\begin{align}
	\Prob\left[ \|\tilde{K} - \tilde{K}_r\|_{\cX_n \times \cX_n} \leq \eps \right] \geq 1 - \delta
	\label{eq:rff-pf:union-bound}
\end{align}
if the rank $r$ is at least
\begin{align}
	r \geq \frac{2}{\eps^2} \log \left( \frac{|\cX_n|^2}{\delta} \right).
	\label{eq:rff-pf:r}
\end{align}
This gives the desired rank bound~\eqref{thm:high-dim-rank} by plugging in the bound $\log |\cX_n| = O(\dim(V) \log (n \dim(V)))$ by Proposition~\ref{prop:norming}, and the definition of $n = O(\sig_K^2/\eps)$.
Moreover, in this success event of~\eqref{eq:rff-pf:union-bound},
\begin{align*}
	\|K - \tilde{K}_r\|_{\cX \times \cX}
	&\leq 
	\|\tilde{K} - \tilde{K}_r\|_{\cX \times \cX} + \eps
	\leq
	\| P_n - Q_n\|_{\cX \times \cX} + 3\eps
	\leq
	2\|P_n - Q_n\|_{\cX_n \times \cX}	+ 3\eps
	\\&\leq
	4\|P_n - Q_n\|_{\cX_n \times \cX_n} + 3\eps
	\leq
	4\|\tilde{K} - \tilde{K}_r\|_{\cX_n \times \cX_n} + 11\eps
	\leq
	15\eps.
\end{align*}
Above, the first inequality is by the uniform approximation~\eqref{eq:rff-pf:K-Ktilde} of $K$ by $\tilde{K}$. The second and penultimate inequalities are by replacing $\tilde{K}$ and $\tilde{K}_r$ with the respective polynomial approximations $P_n$ and $Q_n$, see~\eqref{eq:rff-pf:poly}. The third inequality is because $\cX_n$ is a constant norming set (cf. Proposition~\ref{prop:norming}) and the fact that $P_n(x,y) - Q_n(x,y)$ is a degree $n$ polynomial in $x$ for fixed $y$; and vice versa for the fourth inequality. The final inequality is by the error bound~\eqref{eq:rff-pf:union-bound} on $\cX_n \times \cX_n$. Rescaling $\eps$ by a constant factor $1/15$ concludes the proof.

\section{Examples}\label{sec:ex}

In this section we consider several example varieties. We demonstrate the improved rates for kernel approximation implied by our results by computing the dimension, degree, and Hilbert function for each of these varieties. See Table~\ref{tab:variety} for a summary. We briefly remind the reader of how these three characteristics of varieties arise in our results. 
\begin{itemize}
	\item The \emph{dimension} of the variety is the predominant characteristic for our purposes, as the main point of our kernel approximation results in both the high-precision regime (Theorem~\ref{thm:main-var}) and high-dimensional regime (Theorem~\ref{thm:high-dim}) is that asymptotic dependence on the ambient dimension can be improved to the analogous dependence on the variety's dimension.
	\item The \emph{degree} of the variety is a quantitative measure of the variety's regularity, which is required for kernel approximation (see \S\ref{ssec:intro:cont}). Our high-precision rate depends linearly on it (Theorem~\ref{thm:main-var}), and our high-dimensional rate depends logarithmically on it (Remark~\ref{rem:high-dim-deg}). 
	\item The \emph{Hilbert function} provides tight rank bounds on polynomial kernels over varieties (Proposition~\ref{prop:rank-asymp}). Our high-precision result (Theorem~\ref{thm:main-var}) only uses the asymptotics of this Hilbert function; computing the lower-order terms enables numerical computations and comparisons.
\end{itemize} 
We include explicit computations to illustrate a variety of different techniques for determining these three quantities.

\begin{table}[]
	\centering
	\begin{tabular}{|c|c|c|c|c|}
		\hline
		\textbf{Variety}     & \textbf{Ambient dim} & \textbf{Variety dim} & \textbf{Hilbert function} \bm{$\mathrm{HF}_V(n)$} & \textbf{Where} \\ \hline
		$\R^d$  					& $d$ & $d$ 		&   $\binom{n+d}{d}$  & Ex~\ref{ex:ip-bound}    \\ \hline
		sphere 					& $d$ & $d-1$ 		& $\binom{n+d-1}{d-1} + \binom{n+d-2}{d-1}$ & Ex~\ref{ex:ip-bound}        \\ \hline
		$k$-sparse vectors & $d$ & $k$		& $\sum_{j=0}^k \binom{d}{j}\binom{n}{j}$ & \S\ref{ssec:ex:uas}      \\ \hline
		rank-$1$ matrices		    & $d = m_1m_2$ & $m_1+m_2-1$ 	& $\sum_{k=0}^n \binom{k+m_1-1}{m_1-1}\binom{k+m_2-1}{m_2-1}$  & \S\ref{ssec:ex:lr} 		\\ \hline
		sym. rank-$1$ matrices		    & $d = \binom{m+1}{2}$ & $m$ 	& $\sum_{k=0}^n \binom{2k+m-1}{m-1}$  & \S\ref{ssec:ex:lr} 		\\ \hline
		trig. moment curve 				& $d$ & $1$	     	& $dn+1$ & \S\ref{ssec:ex:moment} 		\\ \hline
		$\SO$ 						& $9$ & $3$ 		& $(2n+3)(2n+1)(n+1)/3$         & \S\ref{ssec:ex:SO3} \\ \hline
	\end{tabular}
	\caption{Characteristics of several example varieties.
	}
	\label{tab:variety}
\end{table}

\subsection{Sparse data}\label{ssec:ex:uas}

Many data-science applications involve $k$-sparse points in a high-dimensional ambient space $\R^d$, where $k \ll d$. These points lie on the algebraic variety $V$ which is the union of all $\binom{k}{d}$ coordinate subspaces of dimension $k$, i.e.,
\begin{align}
	V = \bigcup_{S \subset [d], \, |S| = k} \left\{ x \in \R^d : x_i = 0, \, \forall i \notin S \right\}.
	\label{eq:var-sparse}
\end{align}

\noindent It is clear that $\dim(V) = k$ since each of these $k$-dimensional hyperplanes is an irreducible variety. Following, we also compute the degree and Hilbert function of $V$.

\begin{prop}[Sparse data]\label{prop:sparse:ip}
	Let $V$ be the variety in~\eqref{eq:var-sparse}, and suppose $k < d$.  Then 
	\begin{itemize}
		\item \underline{Hilbert function.} $\HF_{V}(n) = \sum_{j=0}^k \binom{d}{j}\binom{n}{j}$.
		\item \underline{Dimension.} $\dim V = k$.
		\item \underline{Degree.} $\deg V = \binom{d}{k}$.
	\end{itemize}
\end{prop}
\begin{proof}[Proof of Proposition~\ref{prop:sparse:ip}]
	$I(V)$ is the monomial ideal generated by $\{\prod_{i \in S} x_i : S \subseteq [d], |S| = k+1 \}$. Thus $\HF_{V}(n)$ is the number of monomials of degree at most $n$ in $\R[x]$ that are divisible by at most $k$ of $x_1, \dots, x_d$. We count these monomials via casework on the number $j$ of factors.
	For each $j \in \{0, \dots, k\}$, there are $\binom{d}{j}$ choices of the $j$ factors $x_{i_1}, \dots, x_{i_j}$. The corresponding monomials are of the form $\prod_{\ell=1}^j x_{i_\ell}$ times monomials of degree at most $n-j$ in the $j$ variables $x_{i_1}, \dots, x_{i_j}$, of which there are $\binom{n}{j}$ many. 	
	Therefore $\HF_{V}(n) = \sum_{j=0}^k \binom{d}{j}\binom{n}{j}$. Since this is a degree-$k$ polynomial with leading coefficient $\binom{d}{k} / k!$ for $n \geq d$, we have $\dim(V) = k$ and $\deg(V) = \binom{d}{k}$ by Lemma~\ref{lem:hilbert}.
\end{proof}

By Proposition~\ref{prop:rank-asymp}, this Hilbert function computation answers an open question (see the discussion in~\citep[\S3]{ongie2017algebraic}) about tight rank bounds for bounded-degree polynomial kernels that are restricted to sparse data $x,y \in V$. Previously, the only bound which exploited sparsity was $\binom{d}{k}\binom{n+k}{k}$~\citep{CotKesSre11,ongie2017algebraic}.
In contrast, our bound is always at least as good, generically exact (Remark~\ref{rem:rank-asymp:tight}), and sometimes orders-of-magnitude better. 
For example, even for a small-scale instance of sparsity $k=5$ and dimension $d = 100$, the previous rank bound is in the \emph{billions} for degree $n=2$, whereas ours is only about five \emph{thousand}.

\par Note also that this variety $V$ of sparse vectors is not a manifold. This is why kernel approximation methods that exploit manifold structure perform poorly on sparse data, see the discussion in \S\ref{ssec:intro:cont} and Figure~\ref{fig:nystrom-sparse}. In that figure, we run the standard Nystr\"om method using jitter factor $1\mathrm{e}{-10}$ and plot its average performance over $50$ runs. Because there is no closed formula for the $L^{\infty}$ error of Nystr\"om over $\cX = V \cap \ball$, we plot a generous underestimate which evaluates the approximation error at a large number of sampled points. The high-precision method we compare is Taylor Features using our exact rank bound (Remark~\ref{rem:rank-asymp:tight} plus Proposition~\ref{prop:sparse:ip}). We plot its $L^{\infty}$ error which can be computed in closed form.

We conclude this discussion with a numerical illustration in the high-dimensional regime. Since a primary message of this paper is that kernel approximation has a stronger dependence on the variety dimension $k$ than on the ambient dimension $d$, we empirically investigate this in Figure~\ref{fig:rff-sparse} by plotting, for varying $d$ and $k$, the error distribution of RFF over $k$-sparse data in $\R^d$. In this plot, we see qualitatively the mild dependence of the error distribution in $d$, but a stronger dependence in $k$---this is consistent with Theorem~\ref{thm:high-dim}. Note that we plot the error distribution rather than the $L^{\infty}$ error since there is no closed form for the $L^{\infty}$ error of RFF and it requires a prohibitive number of samples to estimate empirically.

\begin{figure}
	\begin{center}
		\includegraphics[width=0.9\textwidth]{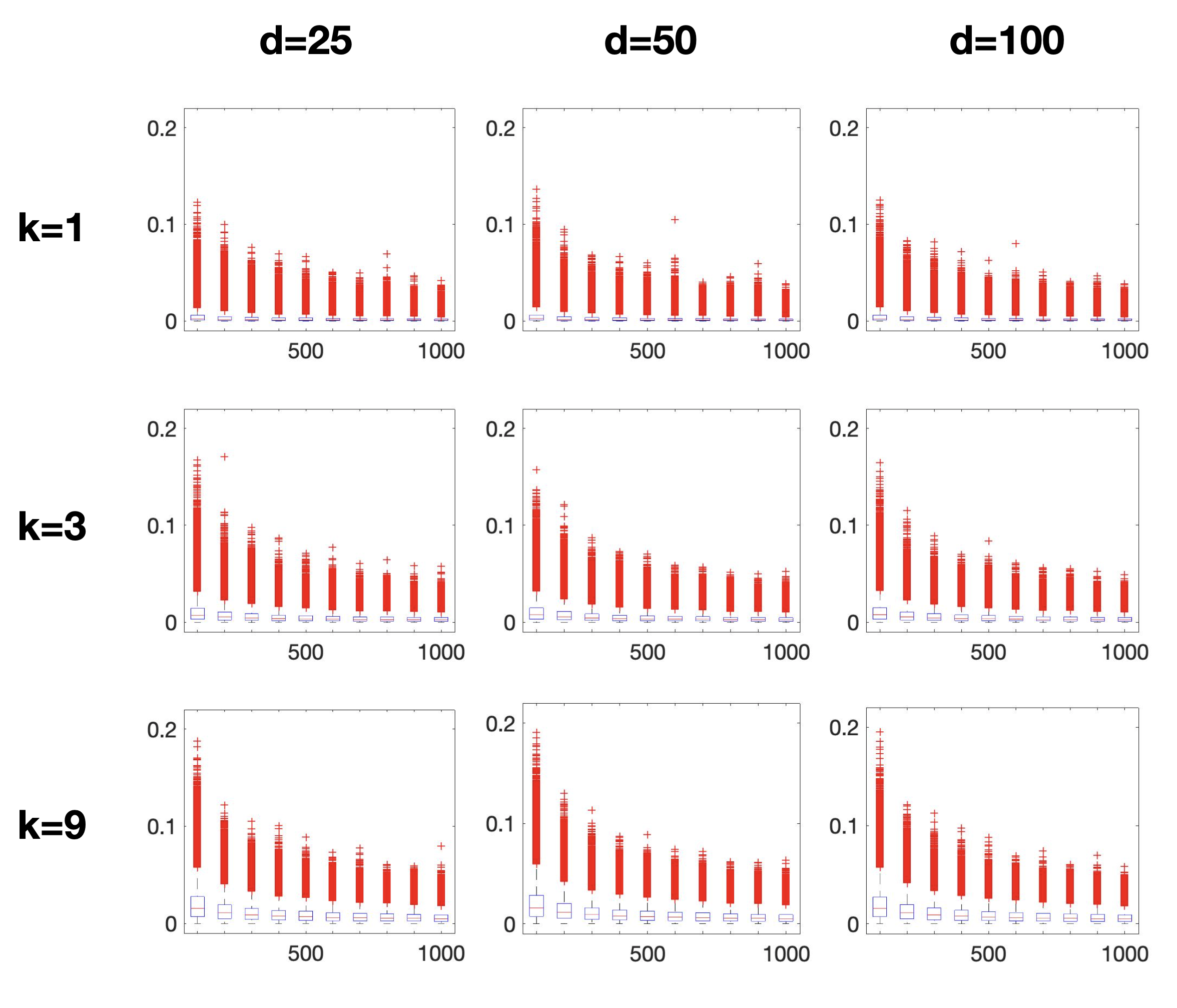}
	\end{center}
	\caption{Error distribution of the RFF approach (see \S\ref{sec:hd}) for approximating the Gaussian kernel $e^{-\|x-y\|^2/2}$ over the variety $V$ of $k$-sparse vectors in $\R^d$, for varying $k$ (rows) and $d$ (columns). In each plot, the $x$-axis is the rank of RFF, and the $y$-axis is a boxplot of the approximation error at $10^6$ pairs of points which are drawn uniformly at random from $V$. From this plot, we see qualitatively the mild dependence of the error distribution in the ambient dimension $d$, but a stronger dependence in the variety dimension $k$---this is consistent with Theorem~\ref{thm:high-dim}.}	
	\label{fig:rff-sparse}
\end{figure}

\subsection{Low-rank matrices}\label{ssec:ex:lr}

Here we consider varieties of low-rank matrices.
For simplicity, we restrict to rank-$1$ matrices; one can perform similar albeit more complicated computations for any fixed rank $r$ since every such set of matrices is a determinantal variety, see e.g.~\citep{harris1984symmetric}.

\par Let us begin with the symmetric case
\begin{align}
	V &= \{xx^T : x \in \R^m\}. \label{eq:rank1-sym}
\end{align}
Note that $V$ is a variety in the set $S^m \cong \R^{\binom{m+1}{2}}$ of symmetric $m \times m$ matrices as it is the vanishing set of all $2 \times 2$ minors (which are quadratic polynomials). As detailed below, the dimension of this variety is $m$, which can be much smaller than the ambient dimension $\binom{m+1}{2}$.

\begin{prop}[Symmetric rank-1 matrices]\label{prop:rank1-sym}
	Let $V$ be the variety in~\eqref{eq:rank1-sym}. Then 
	\begin{itemize}
		\item \underline{Hilbert function.} $\HF_{V}(n) = \sum_{k=0}^n \binom{2k+m-1}{m-1}$.
		\item \underline{Dimension.} $\dim V = m$.
		\item \underline{Degree.} $\deg V = 2^{m-1}$.
	\end{itemize}
\end{prop}
\begin{proof}
	Observe that $I(V)$ is the homogeneous ideal corresponding to the order-$2$ Veronese variety $\nu$ over $(m-1)$-dimensional projective space. Since the projective Hilbert function of this homogeneous ideal is $\binom{2n+m-1}{m-1}$ at degree $n$~\citep[Example 13.4]{harris2013algebraic},
	\[
	\HF_{V}(n) - \HF_{V}(n-1)
	=
	\binom{2n+m-1}{m-1},
	\]
	for $n \geq 1$ by the relation between the affine and projective Hilbert functions of a homogeneous ideal~\citep[Chapter 9, Theorem 12]{Cox13}. Since $\HF_{V}(0) = 1$, telescoping gives the desired Hilbert function identity.
	To compute the dimension and degree of $V$, note that the relation between the affine Hilbert function of $V$ and the projective Hilbert function of $\nu$ implies that the affine Hilbert series of $V$ in indeterminate $t$ is equal to the projective Hilbert series of $\nu$ in $t$, divided by $1-t$. Thus $\dim V = \dim \nu + 1$, and $\deg V = \deg \nu$ by Lemma~\ref{lem:hilbert}. Since $\binom{2n+m-1}{m-1}$ is a polynomial of degree $m-1$ in $n$ with leading coefficient $2^{m-1}/(m-1)!$, Lemma~\ref{lem:hilbert} implies $\dim \nu = m-1$ and $\deg \nu = 2^{m-1}$.
\end{proof}

Next, we consider the variety of non-symmetric rank-$1$ matrices
\begin{align}
	V = \{xy^T : x \in \R^{m_1}, y \in \R^{m_2} \}. \label{eq:rank1-nonsym}
\end{align}
Note that $V$ is a variety in $\R^{m_1 \times m_2} \cong \R^{m_1m_2}$ as it is the vanishing set of all $2 \times 2$ minors. As detailed below, the dimension of this variety is $m_1+m_2-1$, which can be much smaller then the ambient dimension $m_1m_2$. 

\begin{prop}[Rank-1 matrices]\label{prop:rank1}
	Let $V$ be the variety in~\eqref{eq:rank1-nonsym}.
	Then 
	\begin{itemize}
		\item \underline{Hilbert function.} $\HF_{V}(n) = \sum_{k=0}^n \binom{k+m_1-1}{m_1-1}\binom{k+m_2-1}{m_2-1}$.
		\item \underline{Dimension.} $\dim V = m_1 + m_2 - 1$.
		\item \underline{Degree.} $\deg V = \binom{m_1+m_2-2}{m_1-1}$.
	\end{itemize}
\end{prop}
\begin{proof}
	The proof is identical to the proof of Proposition~\ref{prop:rank1-sym}, with the Veronese variety replaced by the order-$2$ Segre variety over the Cartesian product of $(m_1-1)$ and $(m_2-1)$ dimensional projective space. The facts about this Segre variety that are needed are that its projective Hilbert function is $\binom{n+m_1-1}{m_1-1}\binom{n+m_2-1}{m_2-1}$ at degree $n$~\citep[Exercise 13.6]{harris2013algebraic}, from which it is evident that its dimension is $(m_1-1) + (m_2-1)$ and its degree is $\binom{(m_1-1)+(m_2-1)}{m_1-1}$ by Lemma~\ref{lem:hilbert}.
\end{proof}

\subsection{Trigonometric moment curve}\label{ssec:ex:moment}

Here we consider $V$ to be the trigonometric moment curve
\[
M_d =
\left\{ \begin{bmatrix} \cos(\theta), \cos(2 \theta), \dots, \cos(\tfrac{d}{2} \theta), \sin(\theta), \sin(2\theta), 
	\dots, \sin(\tfrac{d}{2}\theta) \end{bmatrix}^T \in \R^{d} \; : \; 
\theta \in [0,2\pi)
\right\}
\]
for $d$ even. (We drop the $0$-th moments since they are constant and thus do not affect the variety's dimension, degree, or Hilbert function.) Although this variety $M_d$ is in ambient dimension $d$, it has dimension $1$. By Proposition~\ref{prop:rank-asymp}, this lets us prove rank bounds for polynomial kernels over $M_d$ that are linear in $d$ rather exponential in $d$. See Figure~\ref{fig:trig} for numerics.

\begin{prop}[Trigonometric moment curve]\label{prop:trig}
	Suppose $d$ is an even integer, and let $V = M_d$.
	\begin{itemize}
		\item \underline{Hilbert function.} $\HF_{V}(n) = nd+1$.
		\item \underline{Dimension.} $\dim V = 1$.
		\item \underline{Degree.} $\deg V = d$.
	\end{itemize}
\end{prop}
\begin{proof}
	Denote $k = d/2$.
	We establish the Hilbert function since it implies the other properties by Lemma~\ref{lem:hilbert}.
	Denote a point in $V$ by $[x,y]^T$ where $x,y \in \R^k$ satisfy $x_j = \cos(j \theta)$ and $y_j = \sin(j \theta)$ for some $\theta \in [0,2\pi)$ and all $j \in [k]$. We claim that the following $(2k-1)k$ quadratic generators form a Gr\"obner basis for $V$ w.r.t. the grlex ordering where $x_1 > \cdots > x_k > y_1 \cdots > y_k$:
	\begin{enumerate}[(i)]
		\item \underline{Square terms.} Take $2x_i^2 - x_{2i} - 1$ and $2y_i^2 + x_{2i} - 1$ for $i = 1, \dots, \lfloor k/2\rfloor$; $2x_i^2 + 2y_k y_{2i-k} - x_{2k-2i} - 1$ and $2y_i^2 - 2y_{k} y_{2i-k} + x_{2k-2i} - 1$ for $i = \lfloor k/2\rfloor + 1, \dots, k-1$; and $x_i^2 + y_i^2 = 1$ for $i=k$.
		\item \underline{$x_ix_j$ cross terms.} Let $i < j$. Take $2x_ix_j - x_{i+j} - x_{j-i}$ for $i+j \leq k$, and $2x_i x_j + 2y_{k} y_{i+j-k} - x_{j-i} - x_{2k-j-i}$ for $i + j > k$. 
		\item \underline{$x_iy_j$ cross terms.} Take $2x_iy_j - y_{i+j} - \sign(i-j)y_{|i-j|}$ for $i \in [k]$, $j \in [k-1]$. 
		\item \underline{$y_iy_j$ cross terms.} Take $2y_iy_j - x_{|i-j|} + x_{i+j}$ for $i < j < k$.
	\end{enumerate}
	Above, $x_0$ and $y_0$ denote $1$ and $0$, respectively. That these polynomials form a Gr\"obner basis is readily checked by observing that each is in $I(V)$ (follows from trigonometric sum-to-product and product-to-sum identities), and that the S-pair criterion holds~\citep{Cox13}.
	\par Thus $\LT(I(V))$ is generated by all quadratic monomials in $S := \{x_1, \dots, x_k, y_1, \dots, y_{k-1}\}$. The corresponding standard monomials of degree at most $n$ are of two types:
	\begin{itemize}
		\item \underline{No factors in $S$}. Then the monomial is in $\R_{\leq n}[y_k]$. 
		 There are $n+1$ such monomials. 
		\item \underline{Single linear factor from $S$}. There are $2k-1$ choices of this factor. 
		The rest of the monomial is in $\R_{\leq n-1}[y_k]$.
		 There are $(2k-1)n$ such monomials total.
	\end{itemize}
	Summing yields $2kn+1$ standard monomials total.
\end{proof}

\begin{remark}[Interpretation via combinatorial algebraic geometry]
	The leading term ideal computed above for the moment curve can be interpreted as the graphical ideal corresponding to the graph on $2k$ vertices $V = \{x_1, \dots, x_k, y_1, \dots, y_k\}$ that is the complete graph with self-loops on $V \setminus \{y_k\}$. 
\end{remark}

\subsection{Rotation matrices}\label{ssec:ex:SO3}

\begin{figure}
	\centering
	\begin{subfigure}{.48\textwidth}
		\centering
		\includegraphics[width=0.75\linewidth]{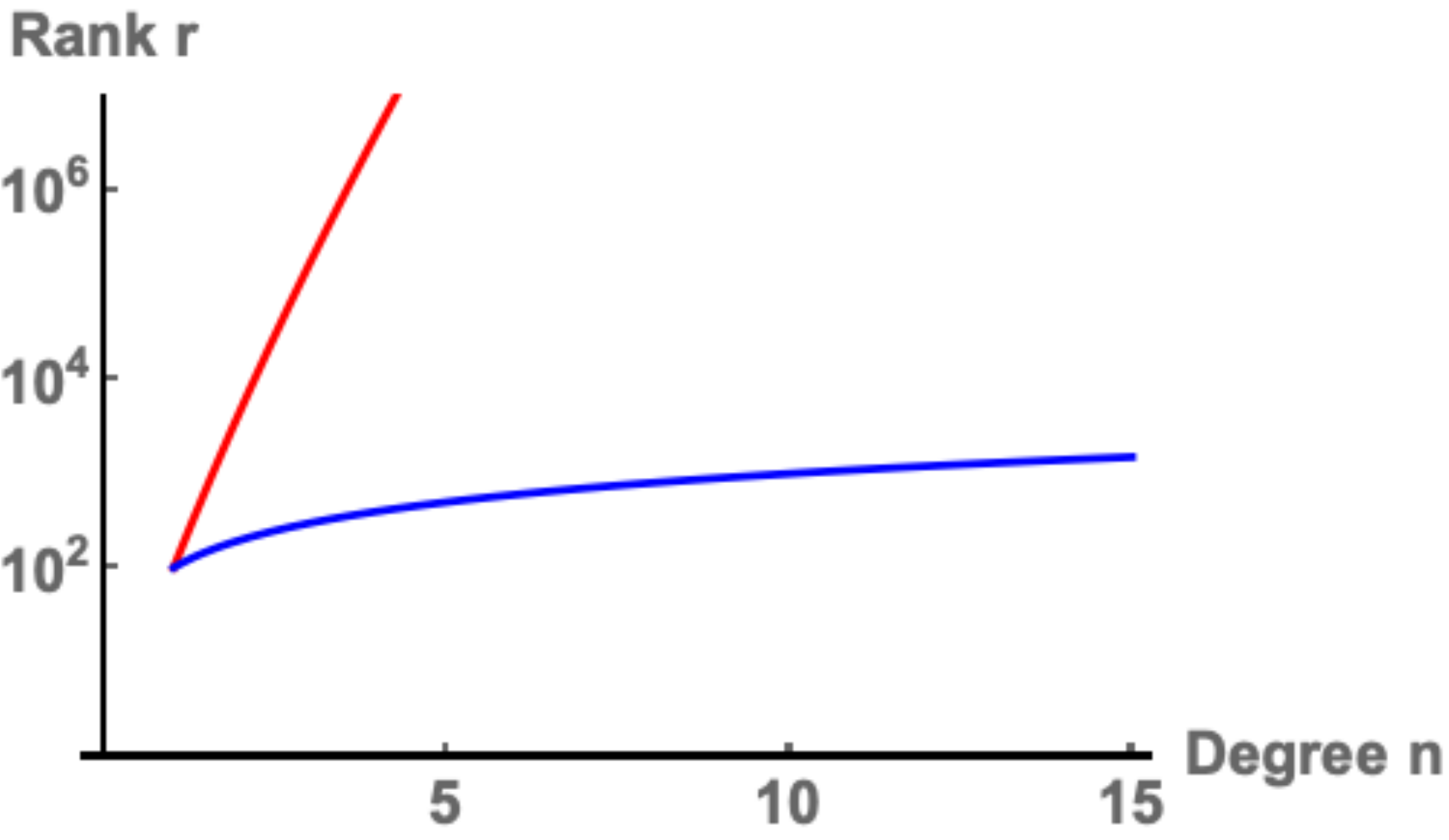}
		\caption{$V$ is the trigonometric moment curve $M_{100} \subset \R^{100}$. The rank is $100n+1$ by Proposition~\ref{prop:trig}, not $\binom{n+100}{n} = O(n^{100})$.}
		\label{fig:trig}
	\end{subfigure}%
	\hfill
	\begin{subfigure}{.48\textwidth}
		\centering
		\includegraphics[width=1\linewidth]{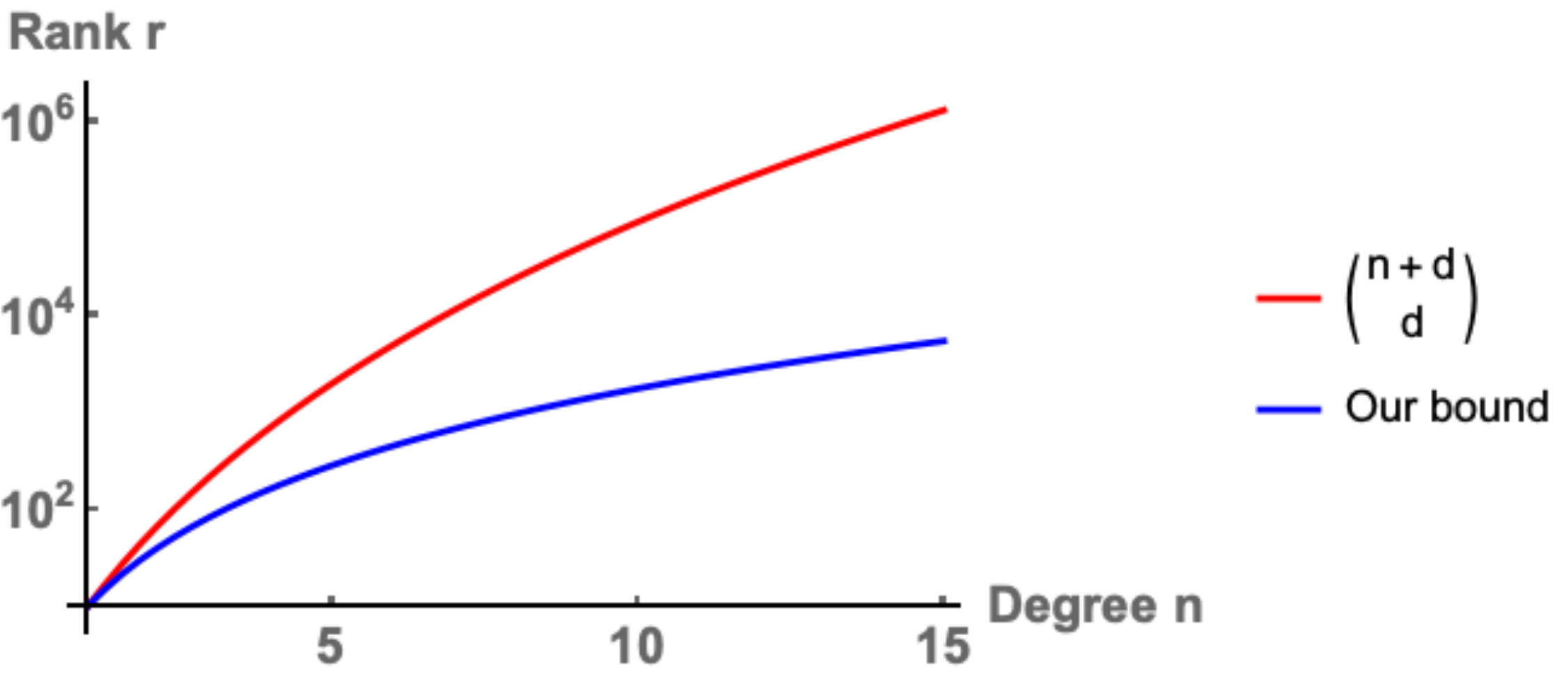}
		\caption{$V$ is the special orthogonal group $\SO \subset \R^9$. The rank is $(2n+3)(2n+1)(n+1)/3$ by Proposition~\ref{prop:SO3}, not $\binom{n+9}{n} = O(n^9)$.}
		\label{fig:so3}
	\end{subfigure}
	\caption{
		Proposition~\ref{prop:rank-asymp} provides improved rank bounds on degree-$n$ polynomial kernels over a variety $V$, namely $\HF_{V}(n) = 
		O(n^{\dim V})$ rather than the standard bound $\binom{n+d}{d} = O(n^d)$. 
		This directly translates into better high-precision rates for kernel approximation (see \S\ref{sec:high-prec}). This improvement is demonstrated here for several varieties $V$. Observe that the $y$-axis is in log-scale.
	}
	\label{fig:ex-rank}
\end{figure}

Here we consider $V$ to be the special orthogonal group 
\[
\SO := \{X \in \R^{3 \times 3} : \det(X) = 1, \; X^TX = 1 \}. 
\]
This is a variety in $\R^{9} \cong \R^{3 \times 3}$ since $\det(X) = 1$ is a polynomial equation in the entries $X_{ij}$, and $X^TX = 1$ is given by $9$ polynomial equations in the entries of $X$. Although this variety is in ambient dimension $9$, its dimension as a variety is significantly smaller, as is intuitively evident by the $3$-dimensional re-parameterization of $\SO$ in terms of the pitch, yaw, and roll scalars. The following proposition makes this precise and computes an tight rank bound for polynomial kernels over $\SO$ that is cubic in their degree. See Figure~\ref{fig:so3} for numerics. See also~\citep{brandt2017degree} for degree computations for higher-order special orthogonal groups.

\begin{prop}[SO(3)]\label{prop:SO3}
	Let $V = \SO$.
	\begin{itemize}
		\item \underline{Hilbert function.} $\HF_{V}(n) = (2n+3)(2n+1)(n+1)/3$.
		\item \underline{Dimension.} $\dim V = 3$.
		\item \underline{Degree.} $\deg V = 8$.
	\end{itemize}
\end{prop}

\section{Discussion}\label{sec:discussion}

We conclude with several interesting directions for future research.

\paragraph*{Exploiting algebraic structure in other problems?} Over the past few decades, exploiting manifold structure has been established as a powerful tool for overcoming the curse of dimensionality throughout machine learning and statistics. This paper shows that one can similarly exploit variety structure (or even approximate variety structure\footnote{Since the approximate kernels in this paper are smooth (they are polynomials of bounded degree or sinuisoids with bounded frequency), they have low error on a neighborhood of the variety.}) in the context of kernel approximation. Can one use the techniques we develop to exploit algebraic structure implicit in datasets in other problems? 
Applications to Optimal Transport will be investigated in forthcoming work.

\paragraph*{Interpolation between high-dimensional and high-precision methods?} Previously, methods in these two categories have been studied in a remarkably disparate way. A first, partial attempt at understanding these two approaches through a common framework is given in \S\ref{ssec:prelim:standard}. However, an understanding of if and how one can gracefully interpolate between these two very different rates remains open. In fact, this tradeoff between better dependence on the error and dimension is poorly understood not just in kernel approximation, but also in other classical fields such as numerical integration (Gaussian vs Monte-Carlo quadrature).

\paragraph*{Exploiting group symmetry in rank bounds?} Many kernels arising in practice enjoy group symmetries such as invariance with respect to coordinate permutations or sign flips. Can this structure be exploited to obtain better rank bounds for polynomial kernels---and thereby better low-rank approximations \`a la our approach in \S\ref{sec:high-prec}?

\paragraph*{Algorithmic questions.} Since the focus of this paper is on theoretical aspects of kernel approximation, our results are primarily existential in nature. Algorithmic questions about how to form these approximations are very interesting and of practical importance---in particular: efficiency, numerical stability, and automatic adaptivity to the variety. While our high-dimensional approach in \S\ref{sec:hd} enjoys these properties since it is based on RFF, these algorithmic questions are more nuanced for the high-precision approach in \S\ref{sec:high-prec} and depend on how the variety is described as input.

\appendix

\section{Polynomial approximation and Fourier decay}\label{ssec:hd-trunc}

Here we show that if a kernel has a compactly supported Fourier transform, then that kernel is well-approximated by a low-degree polynomial. This is a convenient quantitative version of the standard fact that rapid decay in the frequency domain implies smoothness in the natural domain, since given a kernel whose Fourier distribution decays rapidly, one can truncate this Fourier distribution without changing the kernel much, and then approximate this by low-degree polynomials.

\begin{lemma}[Compactly supported Fourier transform implies polynomial approximation]\label{lem:ft-smooth}
	Suppose kernel $K$ satisfies Assumption~\ref{assump:rff}. If its Bochner measure $\mu$ is supported on the ball of radius $r$, then for any $\eps > 0$, there exists a polynomial kernel $P_n$ of degree 
	\[
	n = O\left( r^2 + \log 1/\eps \right)
	\]
	satisfying $\|K - P_n\|_{\ball \times \ball} \leq \eps$.
\end{lemma}

\begin{proof}
	By definition of $\mu$ and then an elementary trigonometric identity,
	\[
	K(x,y) = \E_{\omega \sim \mu} \left[ \cos(\langle \omega, x - y \rangle) \right]
	= 
	\E_{\omega \sim \mu, \theta \sim \Unif([0,2\pi))} \left[ f_{\omega,\theta}(x) f_{\omega,\theta}(y) \right],
	\]
	where we use here the shorthand $f_{\omega,\theta}(x) := \cos(\langle \omega,x \rangle + \theta)$.
	Now for each $\omega$ in the support of $D$ and $\theta \in [0,2\pi)$, there exists a polynomial $p_{\omega,\theta}$ of degree $n = O(r^2 + \log 1/\eps)$ satisfying
	\[
	\sup_{x \in \ball} \abs{f_{\omega,\theta}(x)  - p_{\omega,\theta}(x)} \leq \eps/3.
	\]
	For example, truncating the Taylor series expansion of the cosine function suffices.
	Since the cosine function is bounded in magnitude by $1$, it follows that for all $x,y \in \ball$, 
	\begin{align*}
		\abs{f_{\omega,\theta}(x)f_{\omega,\theta}(y) - p_{\omega,\theta}(x)p_{\omega,\theta}(y)}
		\leq
		\sup_{\substack{z_1,z_2 \in [-1,1] \\ \eps_1, \eps_2 \in [-\eps/3,\eps/3]}} \abs{(z_1 + \eps_1)(z_2 + \eps_2) - z_1z_2} 
		= (1+\eps/3)^2 - 1
		\leq \eps.
	\end{align*}
	Thus the degree-$n$ polynomial kernel
	$
	P(x,y) := \E_{w \sim D, \theta \sim U} \left[ p_{\omega,\theta}(x)p_{\omega,\theta}(y) \right]
	$
	satisfies
	\begin{align*}
		\|K - P\|_{\ball \times \ball} 
		&=
		\sup_{x,y \in \ball} \big| \E_{\omega,\theta} \left[  f_{\omega,\theta}(x) f_{\omega,\theta}(y)  - p_{\omega,\theta}(x)p_{\omega,\theta}(y)\right] \big| \nonumber
		\leq \eps.
	\end{align*}
\end{proof}

\footnotesize
\bibliographystyle{abbrv}
\bibliography{fgt_bib}{}

\end{document}